\pdfoutput=1

\documentclass[journal]{IEEEtran}
\usepackage{xcolor,colortbl,bbm}
\usepackage{algorithmic,algorithm}
\usepackage{graphicx, subcaption}
\usepackage{balance}
\usepackage{amssymb}

\usepackage{cite}

\ifCLASSINFOpdf
\else
\fi
\usepackage{amssymb, amsmath, amsthm,bm,dsfont, float}
\DeclareFontFamily{OT1}{pzc}{}
\DeclareFontShape{OT1}{pzc}{m}{it}{<-> s * [1.10] pzcmi7t}{}
\DeclareMathAlphabet{\mathpzc}{OT1}{pzc}{m}{it}

\newcommand{\col}[1]{\text{col}\left\{#1\right\}}
\newcommand{\grad}[1]{\nabla_{#1^{\tran}} }
\newcommand{\ws}{\bm{{\scriptstyle\mathcal{W}}}}
\newcommand{\wso}{{\scriptstyle\mathcal{W}}}
\newcommand{\wse}{\widetilde{\ws}}
\newcommand{\w}{\bm{w}}

\newcommand{\we}{\widetilde{\w}}
\newcommand{\eqdef}{\:\overset{\Delta}{=}\:}
\DeclareMathOperator*{\argmin}{argmin}

\newcommand{\tran}{{\sf T}}

\newtheorem{theorem}{Theorem}
\newtheorem{assumption}{Assumption}
\newtheorem{lemma}{Lemma}

\newtheorem{definition}{Definition}

\definecolor{Gray}{gray}{0.8}
\definecolor{LightCyan}{rgb}{0.88,1,1}

\makeatletter
\def\endthebibliography{%
	\def\@noitemerr{\@latex@warning{Empty `thebibliography' environment}}%
	\endlist
}
\makeatother

\begin{document}
%
\title{Enforcing Privacy in Distributed Learning with Performance Guarantees}
%
%
%

\author{Elsa~Rizk\IEEEauthorrefmark{1},~\IEEEmembership{Member,~IEEE,}
        Stefan~Vlaski\IEEEauthorrefmark{2},~\IEEEmembership{Member,~IEEE,}
        and~Ali~H. Sayed\IEEEauthorrefmark{1},~\IEEEmembership{Fellow,~IEEE.}
\thanks{\IEEEauthorrefmark{1}The authors are with the School of Engineering, École Polytechnique Fédérale de Lausanne (e-mail: \{elsa.rizk; ali.sayed\}@epfl.ch).
	\IEEEauthorrefmark{2}The author is with the Department of Electrical and Electronic Engineering, Imperial College London (e-mail: s.vlaski@imperial.ac.uk).
}
}

\maketitle

\begin{abstract}
We study the privatization of distributed learning and optimization strategies. We focus on differential privacy schemes
and study their effect on performance.
We show that the popular additive random perturbation scheme degrades performance because it is not well-tuned to the graph structure. For this reason, we exploit two alternative graph-homomorphic constructions and show that they improve performance while guaranteeing privacy. Moreover, contrary to most earlier studies, the gradient of the risks is not assumed to be bounded (a condition that rarely holds in practice; e.g., quadratic risk). We avoid this condition and still devise a differentially private scheme with high probability. We examine optimization and learning scenarios and illustrate the theoretical findings through simulations.

\end{abstract}


%
\IEEEpeerreviewmaketitle

\vspace{-0.2cm}
\section{Introduction}

\IEEEPARstart{D}{istributed} learning and optimization strategies are relevant in many contexts in real-world problems, such as in the design of robotic swarms for rescue missions, or the design of cloud computing services, or the exchange of information over social networks.  Even in scenarios where a centralized solution is possible, it is often preferable to rely on a distributed implementation for various reasons. For instance, the centralized solution tends to have high maintenance costs and is sensitive to the failure of the central processor. Agents may also be reluctant to share their data with a remote central processor due to privacy and safety considerations. The amount of data available at each agent may be significant in size, which makes it difficult to regularly transmit large amounts of data between the dispersed agents and the central processor. Distributed implementations offer an attractive and robust alternative. The architecture can tolerate the failure of individual agents since processing can continue to occur among the remaining agents. Also, agents are only required to share minimal processed information with their neighbours.

There exist several schemes for distributed optimization, which have been studied extensively in the literature. Among these schemes we list the incremental strategy \cite{Bert1997-inc, Blatt2007ACI, Catt2011-inc,lopes2007incremental,John2009-inc,Elias09-inc,rabbat2005quantized,Nedic01-inc}, consensus strategy \cite{DeGroot, Nedic09, xiao2004fast, Barbarossa07,Johansson2008SubgradientMA,Ren05,Boyd06,Kar09,srivastava2011distributed,hlinka2012likelihood}, and diffusion strategy \cite{chen2012diffusion,Tu12,chen2015learning,Chou12,Vlaski21,Lopes08}. The incremental algorithm requires a renumbering of the agents over a cyclic path to cover the entire graph. This is usually a challenging task since the determination of an appropriate cycle is an NP-hard problem and, moreover, the failure of any edge along the path turns the solution moot. 
The consensus and diffusion strategies avoid the need for a circular path over the graph. They rely on the local sharing of information among neighbouring agents. One main difference between both classes of strategies is that consensus updates are asymmetrical, where the starting point for the gradient-descent step is different from the location where the gradient vector is evaluated --- see expression \eqref{eq:conAlg2}. It was shown in several earlier works (see, e.g., \cite{Tu12, sayed2014adaptive, sayed2014adaptation}) that this asymmetry reduces the stability range of consensus implementations in comparison to diffusion solutions, especially in scenarios involving the need for continuous learning and adaptation. 

Now, one key aspect of distributed architectures is that they require agents to share information with their neighbours. This aspect raises an important privacy question about whether the information that is being shared over the edges in the graph can be intercepted. For instance, it is known that in algorithms that rely on gradient-descent updates, information leakage can occur through the sharing of the local gradients or the models that they estimate \cite{Hitaj2017,Melis2019,nasr2019comprehensive,Zhu2019}. This can be problematic when the network is dealing with classified or sensitive data such as healthcare or financial data. In such cases, attackers may be able to recover certain elements of an individual's personal information. There is no question that it is useful to pursue distributed strategies that guarantee a certain level of privacy. 

There exists several useful works in the literature that address privacy questions for distributed algorithms. These contributions rely mainly on two types of tools: differential privacy or cryptography. Cryptographic methods range from using secure aggregation to multiparty computation and homomorphic encryption \cite{bonawitz2016practical,gascon2017privacy,Mohassel2017SecureML,froelicher2021scalable,Niko2013}. Although these methods do not hinder the performance of the learned model, they add significant computational and communication overhead. 

On the other hand, differentially private methods mask the messages by adding some random noise \cite{dwork2014algorithmic, geyer2017differentially,hu2020personalized,triastcyn2019federated,truex2020ldp,wei2020federated,JayDLDP,LiDLDP,ZhuDPDL,pathak2010DP,vlaski2020graphhomomorphic}. They are simple to implement, but they introduce errors into the learned model and reduce the overall utility of the network. One main reason for this degradation is that the noise is often added at will \textit{without} accounting for the graph topology.

In this work, we focus on differential privacy since it is simpler to apply and more scalable. We explain how to adjust its application to match the graph topology, while ensuring privacy and performance guarantees. In particular, we examine the effect of two differentially private schemes: the traditional random perturbations scheme and a graph-homomorphic scheme. We establish the superiority of the latter over the former in the mean-squared-error (MSE) sense. We also devise a third scheme, called \textit{local} graph-homormphic processing, which fully removes the degrading effect of the noise on performance. These results apply to a broad class of distributed learning and optimization formulations.

%

\section{Problem Setup}
We consider a graph topology  with $P$ agents, labelled $p=1,2,\ldots, P$, as illustrated in Fig. \ref{fig:net}. The objective is for the agents to approach the minimizer of an aggregate convex optimization problem of the form:
\begin{equation}\label{eq:optProb}
	w^o \eqdef \argmin_{w \in \mathbb{R}^M} \frac{1}{P}\sum_{p=1}^P\left \{J_p(w) \eqdef  \frac{1}{N_p}\sum_{n=1}^{N_p} Q_p(w;x_{p,n}) \right\},
\end{equation}
where the risk function $J_p(\cdot)$ is associated with the $p$th agent and is defined as an empirical average of the corresponding loss function $Q_p(\cdot;\cdot)$ evaluated at the local data $\{x_{p,n}\}_{n=1}^{N_p}$. 
We associate two non-negative weights $a_{mp}$  and $a_{pm}$ with the edge linking neighbouring agents $m$ and $p$. In this notation, $a_{mp}$ is the weight used by agent $p$ to scale information arriving from $m$, and similarly for $a_{pm}$; it scales information from $p$ toward $m$. The neighbourhood of an agent $p$ is denoted by ${\cal N}_p$ and consists of all agents that are connected to $p$ by an edge. We assume that ${\cal N}_p$ includes agent $p$ as well.  
\begin{figure}[h!]
	\includegraphics[width =0.45\textwidth]{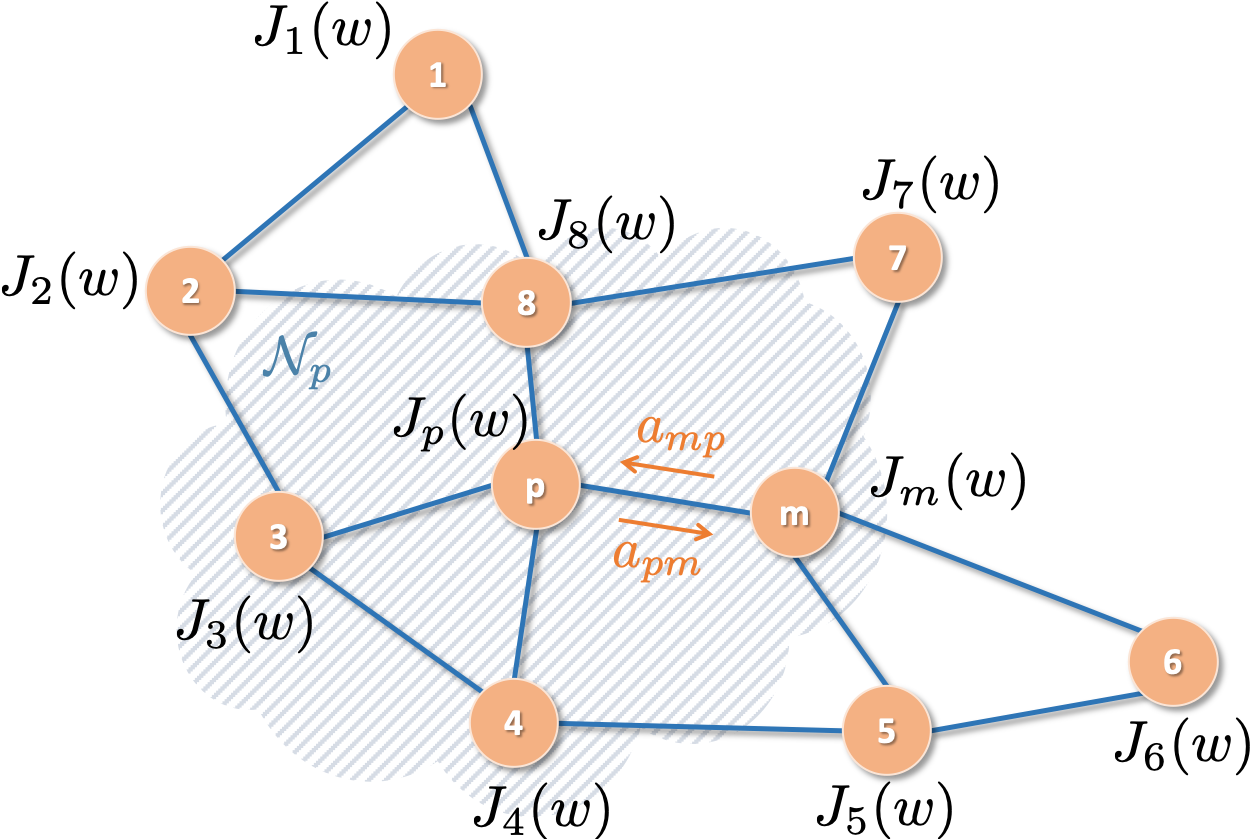}
	\caption{Illustration of a network of agents.}\label{fig:net}
\end{figure}

\subsection{Modeling Conditions}
We assume the individual risk functions $J_p(w)$ are strongly convex and the loss functions $Q_p(w;\cdot)$ have Lipschitz continuous gradients and are twice differentiable. These conditions are common in the study of distributed methods. Although the conditions can be relaxed and the results extended to broader scenarios (see, e.g., \cite{sayed2014adaptation,sayed2014adaptive,ying2018performance,vlaski2021distributed,sayed_2022}), it is sufficient for the purposes of this work to illustrate  the main ideas under these assumptions. 

\begin{assumption}[\textbf{Convexity and smoothness}]\label{ass:fct}
	The risks $J_{p}(\cdot)$ are $\nu-$strongly convex, and the losses $Q_{p}(\cdot;\cdot)$ are convex and twice differentiable, namely for some $\nu > 0$:
	\begin{align} \label{eq:assFctConv}
		&J_{p}(w_2) \geq  \: J_{p}(w_1) + \grad{w}J_{p}(w_1)(w_2-w_1) + \frac{\nu}{2}\Vert w_2 - w_1 \Vert^2,\\
		&Q_{p}(w_2;\cdot) \geq  \: Q_{p}(w_1;\cdot) + \grad{w}Q_{p}(w_1;\cdot) (w_2 - w_1).
	\end{align}
	The loss functions have $\delta-$Lipschitz continuous gradients, meaning there exists $\delta>0$ such that for any data point $x_{p,n}$:
	\begin{equation}\label{eq:assFctLip}
		\Vert \grad{w}Q_{p}(w_2;x_{p,n}) - \grad{w}Q_{p}(w_1;x_{p,n})\Vert \leq \delta \Vert w_2 - w_1\Vert.
	\end{equation}
	\qed
\end{assumption}
\noindent Since we assume the loss functions are twice differentiable, then the above strong-convexity and Lipschitz continuity conditions are equivalent to (see \cite{sayed2014adaptive,sayed2014adaptation,sayed_2022}): 
\begin{equation}
0 < \nu I \leq \grad{w}^2 J_p(w) \leq \delta I.
\end{equation}

We further assume that the graph topology is strongly connected. This means that there exist paths linking any arbitrary pair of agents $(m,p)$ in both directions and, moreover, at least one agent $p$ in the network has a self-loop with $a_{pp}>0$. In other words, at least one agent has some trust in its local information. The combination matrix $A = [a_{mp}]$ is usually left-stochastic meaning that its entries satisfy:
\begin{equation}
a_{mp}\geq0,\;\;\;\;\sum_{m\in{\cal N}_p} a_{mp} = 1.
\end{equation}
That is, the weights on edges connecting agents are nonnegative, and the entries on each column of $A$ add up to one. The strong connectedness of the graph translates into guaranteeing that $A$ is a primitive matrix. As a result, it follows from the Peron-Frobenius theorem \cite{sayed2014adaptation} that $A$ will have a single eigenvalue at one, while all other eigenvalues are strictly inside the unit circle. Moreover, an eigenvector $q$ will exist with positive entries $\{q_p\}$ adding up to one and satisfying: 
\begin{equation}
Aq=q,\;\;\;q_p>0,\;\;\;\mathds{1}^\tran q = 1.
\end{equation}
We refer to $q$ as the Peron eigenvector of $A$. Furthermore, it holds that $\rho\left(A -q\mathbbm{1}^\tran \right ) < 1,$ where $\rho(\cdot)$ denotes the spectral radius of its matrix argument.  

Next, let $w^o$ denote the global minimizer for \eqref{eq:optProb} and let $w_p^o$ denote the local minimizer for $J_p(\cdot)$. We assume that the difference between these global and local models is bounded since, otherwise, collaboration would not be beneficial and one would instead follow a different optimization approach such as multi-task learning \cite{nassif2020multitask}.

 To clarify this point further, we consider a simple example involving a quadratic loss. Assume the  data arriving at node $p$, denoted by $\bm{d}_p(n)$, is generated by some linear regression model under additive noise of the form:
\begin{equation}
\bm{d}_p(n) = \bm{u}_{p,n}^\tran  w^{\star} + \bm{o}_p(n),
\end{equation}
where $\bm{u}_{p,n}$ is the feature vector and $w^{\star}$ is the model. We can seek to estimate $w^{\star}$ by solving:
\begin{equation}
	\min_{w} \frac{1}{P}\sum_{p=1}^P \frac{1}{N_p} \sum_{n=1}^{N_p} ( \bm{d}_p(n)- \bm{u}_{p,n}^{\tran}w )^2.
\end{equation}
The global minimizer in this case is given by:
\begin{equation}\label{eq:regGlobMod}
	w^o = w^{\star} + \widehat{R}^{-1}_u \widehat{r}_{uo},
\end{equation} 
where:
\begin{align}
	\widehat{R}_u &\eqdef \frac{1}{P}\sum_{p=1}^P \left\{ \widehat{R}_{p,u}\eqdef  \frac{1}{N_{p}} \sum_{n=1}^{N_p} \bm{u}_{p,n}\bm{u}_{p,n}^\tran \right\}, \label{eq:dfRu}\\
	\widehat{r}_{uo} & \eqdef \frac{1}{P}\sum_{p=1}^P \left\{ \widehat{r}_{p,uo} \eqdef \frac{1}{N_{p}} \sum_{n=1}^{N_p} \bm{o}_{p}(n)\bm{u}_{p,n} \right\}, \label{eq:dfruv}
\end{align}
while the local minimizers of $J_p(w)$ are given by:
\begin{equation}\label{eq:regLocMod}
	w^o_p = w^{\star} + \widehat{R}^{-1}_{p,u} \widehat{r}_{p,uo}.
\end{equation}
Thus, the global model \eqref{eq:regGlobMod} can be written as a weighted average of the local models \eqref{eq:regLocMod}:
\begin{align}
	w^o = \frac{1}{P}\sum_{p=1}^P \widehat{R}_u^{-1}\widehat{R}_{p,u}w^o_p.
\end{align}
This implies that the global model is a mixture of the local models. Therefore, the bound imposed below on the model difference amounts to an assumption on how  different the distributions of the data across the agents are. This condition is weaker than a uniform bound on the difference between the gradients of the cost functions, which is more commonly assumed in the literature (see  \cite{wang2019adaptive,Li2020On}).

\begin{assumption}[\textbf{Model drifts}]\label{ass:mod}
	The distance of each local model $w_p^o$ to the global model $w^o$ is uniformly bounded, i.e., there exists $\xi\geq 0$ such that $\Vert w^o - w_p^o\Vert \leq \xi$. 
	
	\qed
\end{assumption}

\vspace{-0.5cm}
\section{Distributed Learning}

\subsection{Generalized Distributed Learning}
We focus on two main strategies: consensus and diffusion. The consensus strategy for solving (1) takes the form: 
\begin{align}\label{eq:conAlg}
	\bm{\psi}_{p,i-1} &= \sum_{m\in \mathcal{N}_{p}} a_{m p} \w_{m,i-1}, \\
	\w_{p,i} &= \bm{\psi}_{p,i-1} - \mu \widehat{\grad{w}J_p}(\w_{p,i-1}), \label{eq:conAlg2}
\end{align} 
where $\widehat{\grad{w}J_p}(\cdot)$ denotes a stochastic gradient \textit{approximation} for the true gradient of $J_p(\cdot)$. Usually, the approximation is taken as the gradient of the loss function, namely, $\grad{w} Q_p(\w_{p,i-1}, \bm{x}_{p,i})$. Here, the quantities $\{\bm{\psi}_{p,i},\w_{p,i}\}$ denote estimates for $w^{o}$ at node $p$ at time $i$. Observe that the gradient vector in \eqref{eq:conAlg2} is evaluated at the prior local model $\w_{p,i-1}$ and not at the intermediate model $\bm{\psi}_{p,i-1}$. The diffusion strategy, in turn, admits two related implementations known as combine-then-adapt (CTA) and adapt-then-combine (ATC). They differ by the order in which the calculations are performed with combination coming before adaptation in one case, and with the order reversed in the other case. The CTA diffusion strategy is described by:
\begin{align}
	\bm{\psi}_{p,i-1} &= \sum_{m\in \mathcal{ N}_p} a_{mp}\w_{m,i-1}, \\
	\w_{p,i} &= \bm{\psi}_{p,i-1} - \mu \widehat{\grad{w}J_p}(\bm{\psi}_{p,i-1}). \label{eq:CTAAlg}
\end{align}
Comparing with \eqref{eq:conAlg}--\eqref{eq:conAlg2}, observe now that the starting point in \eqref{eq:CTAAlg} for the gradient-descent step is the same as the point where the gradient vector is evaluated. Similarly, the ATC diffusion strategy is given by: 
\begin{align}\label{eq:ATCAlg}
	\bm{\psi}_{p,i} &= \w_{p,i-1} - \mu \widehat{\grad{w}J_p}(\w_{p,i-1}), \\
	\w_{p,i} &= \sum_{m \in \mathcal{N}_p} a_{mp}\bm{\psi}_{m,i}. 
\end{align}
The above three algorithms can be combined into a single general description as follows \cite{sayed2014adaptation}:
\begin{align} 
	\bm{\phi}_{p,i-1} &= \sum_{m \in \mathcal{N}_p} a_{1,mp} \w_{m,i-1}, \label{eq:preComb}\\
	\bm{\psi}_{p,i} &= \sum_{m \in \mathcal{ N}_p } a_{0,mp} \bm{\phi}_{m,i-1} - \mu \widehat{\grad{w}J_{p}}(\bm{\phi}_{p,i-1}) , \label{eq:update}\\
	\w_{p,i} &= \sum_{m \in \mathcal{N}_p} a_{2,mp} \bm{\psi}_{m,i}. \label{eq:postComb}
\end{align}
where we are introducing three combination matrices, $\{A_0,A_1,A_2\}$. By setting $A_0 = A$ and $A_1=A_2=I$, we obtain consensus, while $A_1=A$ and $A_0=A_2=I$ leads to CTA diffusion, and  $A_2 = A$ and $A_0=A_1=I$ leads to ATC diffusion. Other choices are possible.

\subsection{Privacy Learning}
We now examine differentially private algorithms to safeguard the privacy of the information that is shared among the agents. For illustration purposes, assume the data $\{x_{1,n}\}$ at agent $1$ is replaced by a different set $\{x_{1,n}'\}$. The algorithm will thus take a new trajectory, which we denote by $\{ \bm{\phi}'_{p,i-1}, \bm{\psi}'_{p,i}, \w'_{p,i}\}$. In a private implementation, an external observer should be oblivious to this change at agent $1$. Concretely, all we need to do is add noise to the messages that need privatization. Most commonly, noise with exponential distributions, such as Laplacian or Gaussian, is added \cite{dwork2014algorithmic}. Thus, we are motivated initially to consider a privatized distributed implementation of the following form:
\begin{align}
	\bm{\phi}_{p,i-1} &= \sum_{m\in \mathcal{N}_p} a_{1,mp} \left(\w_{m,i-1} + \bm{g}_{1,mp,i} \right), \label{eq:privPreComb} \\
	\bm{\psi}_{p,i} &= \sum_{m\in\mathcal{N}_p}a_{0,mp} \left( \bm{\phi}_{m,i-1} + \bm{g}_{0,mp,i}\right) - \mu \widehat{\grad{w}J_p}(\bm{\phi}_{p,i-1}), \label{eq:privUpdate}\\
	\w_{p,i} &= \sum_{m\in \mathcal{N}_p}a_{2,mp} \left( \bm{\psi}_{m,i} + \bm{g}_{2,mp,i}\right), \label{eq:privPostComb}
\end{align}
where the $\bm{g}_{j,mp,i}$ denote zero-mean Laplacian random noises for $j=0,1,2$ for every $m,p=1,2,\cdots,P$. For example, in \eqref{eq:privPreComb}, agent $m$ shares $\w_{m,i-1}$ with agent $p$ over the edge that links them. During this transmission, an amount of Laplacian noise $\bm{g}_{1,mp,i}$ is added. The subscript $mp$ is used to denote that this noise is for the directed communication from $m$ to $p$. Similarly, for the other noises. 

We next define differential privacy formally \cite{dwork2014algorithmic}, and show that the above algorithm is indeed differentially private.
\begin{definition}[\textbf{$\epsilon(i)-$Differential privacy}]\label{def:DP}
	We say that the algorithm given by \eqref{eq:privPreComb}$-$\eqref{eq:privPostComb} is $\epsilon(i)-$differentially private for
	agent $p$ at time $i$ if the following condition on the probabilities for observing the respective events holds on the joint distribution $f(\cdot)$ where the notation $\bm{y}_{p,j-1}$ represents any of the shared messages $\{\w_{p,j-1},  \bm{\phi}_{p,j}, \bm{\psi}_{p,j}\}$, $\bm{g}_{\cdot,pm,j}$ the corresponding added noise $\{\bm{g}_{1,pm,j}, \bm{g}_{0,pm,j}, \bm{g}_{2,pm,j}\}$: 
\vspace{-0.1cm}
\begin{align}\label{eq:DPcond}
	\frac{f \left( \Big\{\{ \bm{y}_{p,j-1} + \bm{g}_{\cdot,pm,j}\}_{m\in \mathcal{N}_p\setminus \{p\}} \Big\}_{j=1}^i  \right)}{f \left( \Big\{\{ \bm{y}'_{p,j-1} + \bm{g}_{\cdot,pm,j}\}_{m\in \mathcal{N}_p\setminus \{p\}} \Big\}_{j=1}^i \right) } &\leq e^{\epsilon(i)}.
\end{align} 
\qed
\end{definition}
\noindent The above bounds ensure that for small $\epsilon(i)$, the distributions of the original and modified trajectories are close to each other. This makes it difficult to infer information about the data at the agents since we cannot distinguish the trajectories of the algorithm for different combinations of participating agents. In other words, if agent 1 chooses to replace its original dataset by $\{x'_{1,n}\}$, then the resulting models $\{\w'_{p,j-1},  \bm{\phi}'_{p,j}, \bm{\psi}'_{p,j}\}$ are close enough in distribution to the original models $\{\w_{p,j-1},  \bm{\phi}_{p,j}, \bm{\psi}_{p,j}\}$, and an outside observer will not be able to conclude what dataset was used. The two model trajectories resulting from the use of the original and the alternative dataset are indistinguishable.

To show that algorithm \eqref{eq:privPreComb}$-$\eqref{eq:privPostComb} satisfies condition \eqref{eq:DPcond}, we first calculate the sensitivity of the algorithm. The sensitivity at time $i$ is defined in Appendix \ref{app:sensCalc} as the change in the trajectory of the algorithm if instead of using the original dataset, agent 1 uses the alternative dataset $\{x'_{1,n}\}$. In Appendix \ref{app:sensCalc} the sensitivity is shown to satisfy:
\begin{align}
	\Delta(i) & \eqdef \Vert \ws_i -\ws'_i\Vert \leq B + B' + \sqrt{P}\Vert w^o - w'^o\Vert,  
\end{align}
for some constants $B$ and $B'$ and with high probability. That is, it holds that:
\begin{align}\label{eq:probBd}
	&\mathbb{P}(	\Delta(i) \leq B + B' + \sqrt{P} \Vert w^o - w'^o\Vert ) 
	\notag \\
	&\geq \left ( 1- \frac{\kappa_2^2 \mathds{1}^\tran \Gamma^i \begin{bmatrix}
			\mathbb{E}	\Vert \bar{\ws}_0\Vert^2  \\ \mathbb{E}\Vert \check{\ws}_0\Vert^2 
		\end{bmatrix} + O(\mu)+O(\mu^{-1})}{B^2} \right)
	\notag 	\\
	& \times \left(1-  \frac{\kappa'^2_2 \mathds{1}^\tran \Gamma'^i \begin{bmatrix}
			\mathbb{E}	\Vert \bar{\ws}'_0\Vert^2  \\ \mathbb{E}\Vert \check{\ws}'_0\Vert^2 
		\end{bmatrix} + O(\mu) +O(\mu^{-1})}{B'^2} \right),
\end{align}
where $\ws_i \eqdef \col{\w_{p,i}}_{p=1}^P$, the model error at time zero is denoted by $\widetilde{\ws}_0 \eqdef \col{w^o-\w_{p,0}}_{p=1}^P$, and the variables $\{\bar{\ws}_0 ,\check{\ws}_0\}$ arise from the partitioning   $\mathcal{V}_{\theta}^\tran \widetilde{\wso}_0 = \text{col}\{ \bar{\ws}_0, \check{\ws}_0\}$ with the matrix $\Gamma$ and the constant $\kappa_2$ defined in Appendix \ref{app:MSEpriv}. Result \eqref{eq:probBd} means that the sensitivity $\Delta(i)$ is bounded with high probability. The bound constants $B$ and $B'$ are chosen by the user: larger values for $B$ and $B'$ result in higher probability of bounded sensitivity but, as shown in \eqref{eq:eps}, they result in a larger privacy bound. In other words, the values of $B$ and $B'$ can be controlled to balance the trade-off between the privacy level and the likelihood of bounded sensitivity. Next, if we denote the variance of the Laplacian noise $\bm{g}_{j,mp,i}$ by $\sigma_{g}^2$,
with $\bm{y}_{p,j-1} = \w_{p,j-1}$ the fraction in \eqref{eq:DPcond} can be bounded as follows with high probability:
\begin{align}
	&\frac{f \left( \Big\{\{ \w_{p,j-1} + \bm{g}_{0,pm,j}\}_{m\in \mathcal{N}_p\setminus \{p\}} \Big\}_{j=1}^i  \right)}{f \left( \Big\{\{ \w'_{p,j-1} + \bm{g}_{0,pm,j}\}_{m\in \mathcal{N}_p\setminus \{p\}} \Big\}_{j=1}^i \right) } \notag \\
	&\stackrel{(a)}{=} \prod_{j=1}^i\frac{ f\left( \{ \w_{p,j-1} + \bm{g}_{0,pm,j}\}_{m\in \mathcal{N}_p\setminus \{p\}}  | \mathcal{X}_{j-1}\right)}{f\left( \{ \w'_{p,j-1} + \bm{g}_{0,pm,j}\}_{m\in \mathcal{N}_p\setminus \{p\}}  | \mathcal{X}'_{j-1}\right)}
	\notag \\
	&\stackrel{(b)}{=} \prod_{j=1,m\in\mathcal{N}_p\setminus \{p\}}^i \frac{\exp \left(-\sqrt{2} \Vert\w_{p,j-1} + \bm{g}_{0,pm,j} \Vert/\sigma_g \right)}{ \exp \left(-\sqrt{2} \Vert\w'_{p,j-1} + \bm{g}_{0,pm,j} \Vert/\sigma_g\right)} \notag \\
	&\leq \exp \bigg( \frac{\sqrt{2}}{\sigma_g}\sum_{j=1, m\in\mathcal{N}_p\setminus \{p\}}^i \Vert\w_{p,j-1} -  \w'_{p,j-1}\Vert \bigg) \notag \\
	&\leq \exp \bigg( \frac{\sqrt{2}P}{\sigma_g}\sum_{j=1}^i \Vert\ws_{j-1} -  \ws'_{j-1}\Vert \bigg) ,
\end{align} 
where the first equality $(a)$ follows from applying Bayes' rule with $\mathcal{X}_{j-1} \eqdef \{\w_{p,j-1}\} \cup \Big\{\{ \w_{p,o-1} + \bm{g}_{0,pm,o}\}_{m\in \mathcal{N}_p\setminus \{p\}} \Big\}_{o=1}^{j-1}$, and the second equality $(b)$ follows from the independence of $\w_{p,j-1} + \bm{g}_{0,pm,j}$ for $m\in \mathcal{N}_p\setminus \{p\}$ conditioned on $\w_{p,j-1}$. 
A similar bound can be found for $\bm{y}_{p,j-1} \in \{\bm{\phi}_{p,j},\bm{\psi}_{p,j}\}$.

Thus, the level of privacy is defined by the following choice for $\epsilon(i)$ in terms of the running $\Delta (j)$ values:
\begin{align}
	\epsilon(i) &= \frac{\sqrt{2}P}{\sigma_{g}} \sum_{j=0}^{i-1} \Delta(j) 
	\leq \frac{\sqrt{2}P}{\sigma_{g}} (B + B' + \sqrt{P} \Vert w^o - w'^o\Vert )i.
	 \label{eq:eps}
\end{align}
These results show that in order to arrive at an $\epsilon(i)-$differentially private algorithm, it is sufficient to select the variance of the Laplacian noise to satisfy \eqref{eq:eps}. Expression \eqref{eq:eps} shows that $\epsilon(i)$ is a linear function of the iterations. This means that the process becomes less private at a rate no greater than a linear rate. It is important to note here that most earlier studies on differentially private schemes for multi-agent systems \cite{hu2020personalized,wei2020federated,JayDLDP} assume bounded gradients for the risk function. However, this condition is rarely satisfied in practice. For instance, even quadratic risks have unbounded gradients. For this reason, in our approach, we have avoided relying on this assumption. Instead, we are able to establish that differential privacy continues to hold with high probability. 

We still need to examine the effect of the added noises on performance. To do so, we introduce the extended model $\ws_i \eqdef  \col{\w_{p,i}}_{p=1}^P$ and write the three-step algorithm \eqref{eq:privPreComb}--\eqref{eq:privPostComb} using one single recursion as follows: 
\begin{align}\label{eq:extRec}
	\ws_i = &\mathcal{A}_2^\tran \mathcal{A}_0^\tran \mathcal{A}_1^\tran \ws_{i-1} - \mu\mathcal{A}_2^\tran \col{\widehat{\grad{w}J_p}(\bm{\phi}_{p,i-1})}_{p=1}^P \notag \\
	& + \mathcal{A}_2^\tran \mathcal{A}_0^\tran \text{diag}(\mathcal{A}_1^\tran \bm{\mathcal{G}}_{1,i}) + \mathcal{A}_2^\tran 	\text{diag}(\mathcal{A}_0^\tran\bm{\mathcal{G}}_{0,i})  \notag \\
		&+\text{diag}(\mathcal{A}_2^\tran \bm{\mathcal{G}}_{2,i}),
\end{align}
where for $j=0,1,2$, $\mathcal{A}_j \eqdef A_j \otimes I_M$, and $\bm{\mathcal{G}}_{j,i}$ is a matrix whose entries are the added noises $\bm{g}_{j,mp,i}$. We denote the model error by $\wse_i \eqdef \col{w^o - \w_{p,i}}_{p=1}^P$, and introduce the local gradient noise:
\begin{equation}\label{eq:gradNoise}
	\bm{s}_{p,i} \eqdef \widehat{\grad{w}J_{p}} (\bm{\phi}_{p,i-1}) - \grad{w}J_p(\bm{\phi}_{p,i-1}).
\end{equation}
It is customary to assume that this gradient noise process has zero mean and bounded second-order moment (see, e.g., \cite{sayed2014adaptation, sayed_2022}, where this property is actually shown to hold in many important cases of interest and similar arguments can be applied to the current case), namely:
\begin{equation}\label{eq:gradNoiseBD}
	\mathbb{E} \{ \Vert \bm{s}_{p,i}\Vert^2 | \mathcal{F}_{i-1}\} \leq \beta_{s,p}^2 \Vert \widetilde{\bm{\phi}}_{p,i-1}\Vert^2 + \sigma_{s,p}^2, 
\end{equation}
for some nonnegative constants $\beta_{s,p}^2$ and $\sigma_{s,p}^2$, and where the conditioning is taken over all past models $\mathcal{F}_{i-1} \eqdef \text{filtration}\{\w_{p,j}\}_{p=1,j=0}^{P,i-1}$. Then, using the extended gradient noise $\bm{s}_i \eqdef \col{\bm{s}_{p,i}}_{p=1}^P$, the error recursion corresponding to \eqref{eq:extRec} is given by:
\begin{align}\label{eq:extErrRec}
	\wse_i = &\mathcal{A}_2^\tran\mathcal{A}_0^\tran \mathcal{A}_1^\tran   \wse_{i-1} + \mu \mathcal{A}_2^\tran \col{\grad{w}J_p(\bm{\phi_{p,i-1}})}_{p=1}^P  \notag \\
	& + \mu \mathcal{A}_2^\tran \bm{s}_i - \mathcal{A}_2^\tran \mathcal{A}_0^\tran \text{diag}( \mathcal{A}_1^\tran \bm{\mathcal{G}}_{1,i}) - \mathcal{A}_2^\tran \text{diag}( \mathcal{A}_0^\tran\bm{\mathcal{G}}_{0,i}) 
	\notag \\
	& -\text{diag}( \mathcal{A}_2^\tran \bm{\mathcal{G}}_{2,i}).
\end{align}
Since $J_p(\cdot)$ are twice differentiable, we appeal to the mean-value theorem to express the gradient in the form \cite{sayed2014adaptation}:
\begin{align}\label{eq:MVT}	\grad{w}J_p(\bm{\phi}_{p,i-1}) = - \bm{H}_{p,i-1} \widetilde{\bm{\phi}}_{p,i-1} - \grad{w}J_p(w^o),
\end{align}
where:
\begin{equation}
	\bm{H}_{p,i-1} \eqdef \int_0^1 \grad{w}^2 J_p(w^o - t\bm{\phi}_{p,i-1})dt.
\end{equation}
Then, introducing the quantities:
\begin{align}
	\bm{\mathcal{B}}_{i-1} &\eqdef \mathcal{A}_2^\tran(\mathcal{A}_0^\tran - \mu \bm{\mathcal{H}}_{i-1})\mathcal{A}_1^\tran, \\
	\bm{\mathcal{H}}_{i-1} &\eqdef \text{diag}\{\bm{H}_{p,i-1}\}_{p=1}^P, \\
	b &\eqdef \col{\grad{w}J_p(w^o)}_{p=1}^P, \label{eq:defb}
\end{align}
we rewrite \eqref{eq:extErrRec} as:
\begin{align}\label{eq:recPrivDist}
	\wse_i &=  \bm{\mathcal{B}_{i-1}} \wse_{i-1} + \mu \mathcal{A}_2^\tran \bm{s}_i - \mu \mathcal{A}_2^\tran b  + \mu \mathcal{A}_2^\tran \bm{\mathcal{H}}_{i-1}\text{diag}(\mathcal{A}_1^\tran \bm{\mathcal{G}}_{1,i}) \notag \\
	& - \text{diag}(\mathcal{A}_2^\tran\bm{\mathcal{G}}_{2,i}) - \mathcal{A}_2^\tran \text{diag}(\mathcal{A}_0^\tran \bm{\mathcal{G}}_{0,i}) - \mathcal{A}_2^\tran \mathcal{A}_0^\tran \text{diag}(\mathcal{A}_1^\tran\bm{\mathcal{G}}_{1,i}).
\end{align}


We show in the next theorem that the weight-error size converges to the neighbourhood of zero, with the size of the neighbourhood determined by the step-size and the added noise variance. 
\begin{theorem}[\textbf{MSE convergence of privatized distributed learning}]\label{thrm:MSEpriv}
	Under assumptions  \ref{ass:fct} and \ref{ass:mod}, the distributed recursions \eqref{eq:privPreComb}$-$\eqref{eq:privPostComb} converge exponentially fast for a small enough step-size to a neighbourhood of the optimal model, i.e.:
	\begin{align} \label{eq:thrmMESpriv}
		& \limsup_{i\to \infty} \mathbb{E}\Vert \wse_i\Vert^2
		\leq 
		O(\mu)\sigma_s^2 + O(\mu) + (O(\mu^{-1})+O(\mu))\sigma_{g}^2.
	\end{align}
\end{theorem}
\begin{proof}
	See Appendix \ref{app:MSEpriv}.
\end{proof}
By examining the bound in \eqref{eq:thrmMESpriv} on the mean-square error (MSE), we observe that the noise variance $\sigma_g^2$ appears multiplied by a term on the order of $\mu^{-1}$, which is detrimental to performance when $\mu$ is small. Therefore, the traditional approach of adding Laplacian noise over the edges to ensre privacy is \textit{calamitous} to performance and needs to be improved. We describe next an alternative approach. 

\subsection{Graph-Homomorphic Noise}
The noises added to the communication links in the previous section did not take into account the graph topology. As a result, their effect gets magnified by $O(\mu^{-1})$ as shown in \eqref{eq:thrmMESpriv}. We now examine another strategy for adding noise, which relies on a graph-homomorphic construction from \cite{vlaski2020graphhomomorphic}. Specifically, the noises are now constructed to satisfy the following condition: 
\begin{align}\label{eq:GHP-cond}
	\sum_{p,m=1}^P q_p a_{mp} \bm{g}_{j,mp,i} &= 0, 
\end{align}
for $j = 0,1,2$, and where $q = \col{q_p}_{p=1}^P$ is the Perron eigenvector of $A_2^\tran A_0^\tran A_1^\tran$. This can be satisfied if we continue to choose zero-mean Laplacian noises $\bm{g}_{j,p,i}$ with variance $\sigma_{g}^2$ and then set:
\begin{equation}\label{eq:GHP-noise}
	\bm{g}_{j,pm,i} = \begin{cases}
		\frac{a_{pm}}{a_{mp}}	\bm{g}_{j,p,i}, & m \neq p \\
			- \frac{1-a_{j,pp}}{a_{j,pp}} \bm{g}_{j,p,i},  & m = p.
		\end{cases}
\end{equation} 
Condition \eqref{eq:GHP-cond}, along with construction \eqref{eq:GHP-noise}, ensure that the net effect of the additional noises cancel out over the entire graph during the local aggregation steps. We show in the next theorem that the MSE bound improves in this case. To see this, we first introduce the network centroid $\w_{c,i}$ and study its convergence under graph-homomorphic perturbations. Let:
\begin{align}\label{eq:netCent}
	\w_{c,i} &\eqdef \sum_{p=1}^P q_p \w_{p,i}
	\notag \\
	&= \w_{c,i-1}    - \mu \sum_{p=1}^P q_p\bm{s}_{p,i} 
	\notag \\
	&\quad  - \mu \sum_{p=1}^P q_p \grad{w}J_p\left(\sum_{m\in \mathcal{N}_p} a_{1,mp}(\w_{m,i-1} + \bm{g}_{1,mp,i})\right)
	\notag \\
	&\quad 
	+ \sum_{p,m} q_p \left( a_{1,mp} \bm{g}_{1,mp,i} + a_{0,mp}\bm{g}_{0,mp,i} + a_{2,mp}\bm{g}_{2,mp,i} \right). 
\end{align} 
Since we are using graph-homomorphic perturbations, the sum of the noise terms in the last line cancels out. We can therefore write the following error recursion:
\begin{align}\label{eq:netCent-err}
	\we_{c,i}  & = \we_{c,i-1} +  \mu (q^\tran \otimes I) \bm{s}_i +  \mu (q^\tran \otimes I) b
	\notag \\
	&\quad - \mu\sum_{p=1}^P q_p \bm{H}_{p,i-1} \sum_{m \in \mathcal{N}_p} a_{1,mp} (\we_{m,i-1} - \bm{g}_{1,mp,i} ).
\end{align}

Before stating the theorem on the MSE convergence, we bound the network disagreement defined as the average second-order moment of the difference between the local models and the centroid model.

\begin{lemma}[\textbf{Network disagreement}]\label{lem:netDis}
	The average deviation
	from the centroid is uniformly bounded during each iteration i, and, moreover, it holds asymptotically that:
	\begin{align}\label{eq:lemNetDis}
		&\limsup_{i\to\infty} \frac{1}{P}\sum_{p=1}^P \mathbb{E}\Vert \w_{p,i} - \w_{c,i}\Vert^2  
		\leq O(1)\sigma_g^2 + O(\mu)
	\end{align}
\end{lemma}
\begin{proof}
	See Appendix \ref{app:netDis}.
\end{proof}
Expression \eqref{eq:lemNetDis} shows that the local models will be at most $O(1)\sigma_g^2$ away from the centroid model. Thus, if the centroid model manages to converge to the optimal model $w^o$ with only a slight variation, then the local models will always be a constant, proportional to the noise variance $\sigma_g^2$, away from the true model. In the next theorem, we show that the added noise only alters the centroid model by an $O(1)$ factor.

\begin{theorem}[\textbf{MSE convergence of the network centroid}]\label{thrm:MSE-netCent}
	Under assumptions \ref{ass:fct} and \ref{ass:mod}, the network centroid defined in \eqref{eq:netCent} converges exponentially fast for a small enough step-size to a neighbourhood of the optimal model:
	\begin{align}
		\limsup_{i \to \infty} \mathbb{E}\Vert \we_{c,i}\Vert^2 \leq&  \:
		O(\mu)\sigma_s^2 + O(1)\sigma_g^2 + O(\mu^2).
	\end{align}
\end{theorem}
\begin{proof}
	See Appendix \ref{app:thrmMSE-netCent}.
\end{proof}
Thus, the network centroid is at most $O(1)\sigma_g^2$ away from the true minimizer $w^o$, even with added noise, as opposed to $O(\mu^{-1})\sigma_g^2$. In Lemma \ref{lem:netDis}, we showed that the individual models $\w_{p,i}$ are $O(1)\sigma_{g}^2$ away from the centroid model. Thus, by using the graph-homomorphic perturbations \eqref{eq:GHP-cond}--\eqref{eq:GHP-noise}, the MSE is not inversely proportional to $\mu$ anymore, which is an improvement relative to \eqref{eq:thrmMESpriv}.

\subsection{Local Graph-Homomorphic Noise}
We explain how to improve on the $O(1)\sigma_g^2$ deviation and replace it by $O(\mu)\sigma_g^2$, by relying on the use of {\em local} graph-homomorphic noise. To do so, we construct the noises to satisfy the following alternative condition as opposed to \eqref{eq:GHP-cond}:
\begin{equation}\label{eq:locGHPcond}
	\sum_{m\in \mathcal{ N}_p} a_{mp} \bm{g}_{j,mp,i} = 0.
\end{equation}
Observe that we are requiring the sum of the noises to cancel out {\em locally}, rather than globally as required in the previous section. The neighbours of every agent $p$ must collaborate together to generate dependent random noises that will cancel out locally at $p$. The collaboration will occur through agent $p$, since a direct link might not exist amongst these neighbours. A similar problem exists in blockchain applications where the generation of a random number is required to occur in a distributed manner \cite{ginar_2019}.

We now devise a distributed scheme that leads to noises that satisfy condition \eqref{eq:locGHPcond}. For the sake of demonstration, we describe the protocol through an example. Thus, assume agent $1$ is connected to 5 agents labeled $2,3,4,5,6$ (Fig. \ref{fig:locGHP}, \textit{left}). Since the neighbours of agent $1$ need not be connected to each other through direct links, all communications will take place through agent $1$. In this subnetwork, we allow agent $1$ to be the orchestrator of the scheme. The first step is for agent $1$ to split its neighbours into two disjoint sets $\mathcal{N}_1 = \mathcal{N}_{+}\bigcup \mathcal{N}_{-}$. For example, we may collect the even numbered agents into $\mathcal{N}_{+}$, and the odd numbered agents into $\mathcal{N}_{-}$. Then, we allow every pair of agents from the two disjoint sets to agree on a noise value they will add to their message such that they will cancel out at agent $1$. We force agents from $\mathcal{N}_{-}$ to multiply the noise they will add to their messages by a negative sign. Therefore, agent $2$ will add to its message two noise terms, one generated with agent $3$ and another with agent $5$. We denote the noise term generated by agents $2$ and $3$ that will be sent to agent $1$ by $\bm{g}_{\{23\}1,i}$. Since the messages are scaled by the weights attributed by agent $1$ to its neighbours, the added noise must then be divided by the weights, i.e., the message sent by agent 2 to agent 1 is the original message $\w_{2,i}$ and the two generated noises by agent 2 with agents 3 and 5 scaled by the corresponding weight $a_{12}$:
\begin{equation}
	\w_{2,i} + \frac{ \bm{g}_{\{23\}1,i}}{a_{12} }+ \frac{\bm{g}_{\{25\}1,i}}{a_{12}}.
\end{equation}
However, this requires that agent 2 know the weight attributed to its messages by agent 1. Thus, agent 1 will have to make the weights public in case of a non-doubly stochastic combination matrix. The same process occurs between agents 4 and 6 with both 3 and 5. Then, the aggregate messages sent to agent 1 will end up being the sum of the unmasked weights:
\begin{align}
	&\sum_{k \in \mathcal{N}_+} a_{1k}\bigg( \w_{k,i} + \sum_{\ell \in \mathcal{N}_{-}} \frac{\bm{g}_{\{k\ell\}1,i}}{a_{1k}} \bigg) \notag \\
	&+ \sum_{k \in \mathcal{N}_-} a_{1k} \bigg( \w_{k,i} -  \sum_{\ell \in \mathcal{N}_{+}} \frac{\bm{g}_{\{\ell k\}1,i}}{a_{1k}} \bigg)
	= \sum_{k \in \mathcal{N}_1} a_{1k}\w_{k,i}.
\end{align}

We move to the method used to generate the pairwise noise terms $\bm{g}_{\{k\ell\}1,i}$. We rely on the Diffie-Helman key exchange protocol where each pair of agents shares a secret key that is used to generate the added noise. Given two agents, say 2 and 3, we assume they have individual secret keys $v_2$ and $v_3$, respectively. A known modulus $\pi$ and base $b$ is agreed upon amongst the agents. Then, agent 2 broadcasts its public key $V_2 = (b^{v_2} \mod \pi)$ and agent 3 does the same $V_3 = ( b^{v_3} \mod \pi)$. Agent 2 then calculates, $v_{23} = (V_3^{v_2} \mod \pi)= (b^{v_2 v_3} \mod \pi)$ which is the same as what agent 3 calculates $v_{23} = (V_2^{v_3} \mod p) = (b^{v_3 v_2} \mod \pi)$. Thus, the two agents now share a secret key $v_{23}$ only known to them. This secret key can then be used as the added noise to mask the messages, i.e., $\bm{g}_{\{23\}1,i} = v_{23}$. However, to make the process differentially private we need the resulting added noise to be Laplacian, $\text{Lap}(0,\sigma_g/\sqrt{2})$. In what follows, we describe a scheme, which we call the local graph-homomorphic processing scheme that ensures the added noise is Laplacian. An illustration of this process is found in the right subfigure of Fig. \ref{fig:locGHP}. 

\begin{figure}[h!]
	\includegraphics[width =0.5\textwidth]{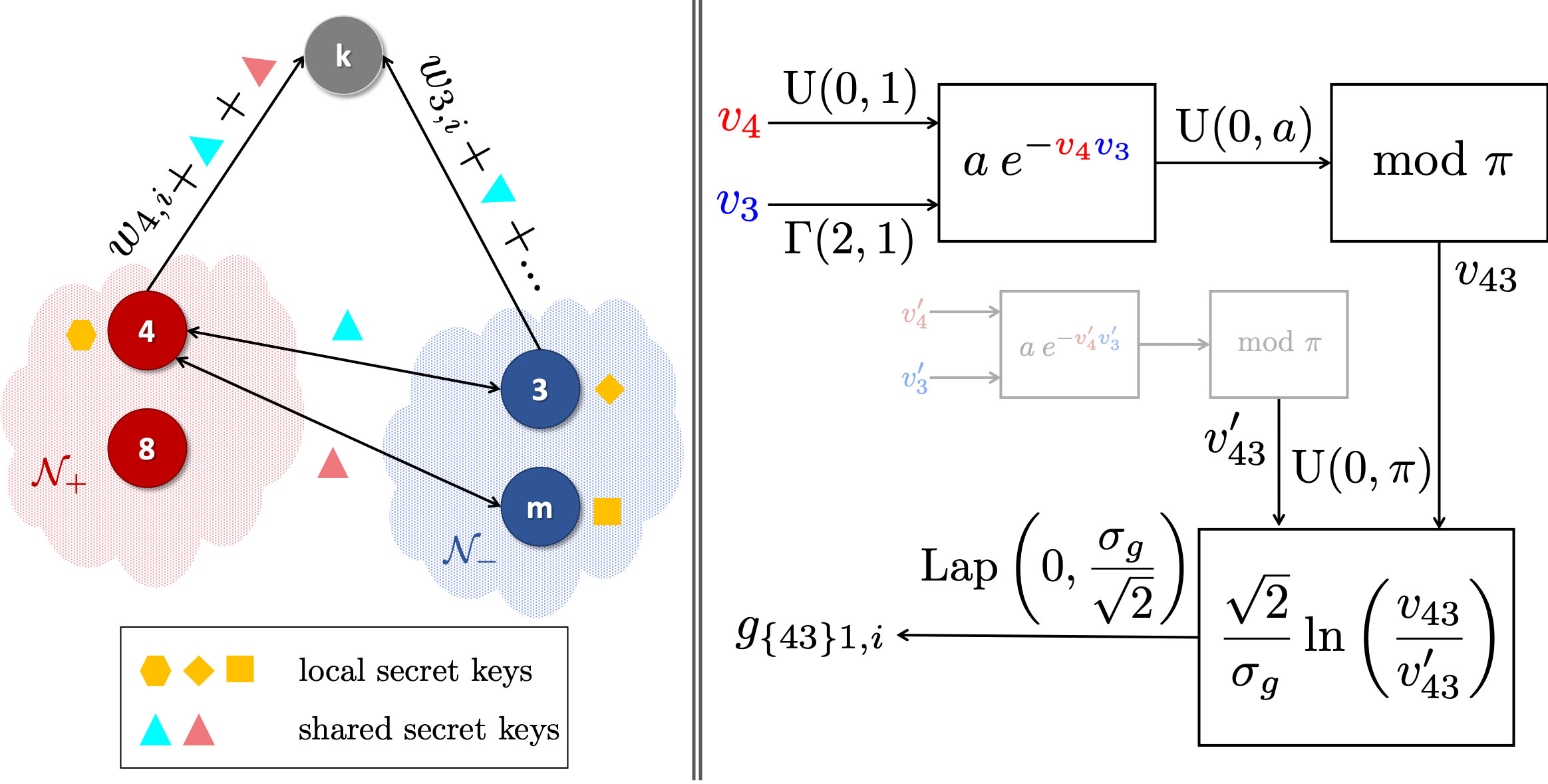}
	\caption{Illustration of the local graph-homomorphic process. The figure on the left describes the Diffie-Helman key exchange procedure. The figure on the right shows the transformation the random variable goes through.}\label{fig:locGHP}
\end{figure}

\begin{definition}[\textbf{Local graph-homomorphic process}]\label{def:locGHP}
	We are given a subnetwork of agent $k$, and neighbours $\ell \in \mathcal{N}_{+}$ and $m \in \mathcal{N}_{-}$. Let agent $\ell$ sample two secrect keys $\bm{v}_{\ell}$ and $\bm{v}'_{\ell}$ from a uniform distrubution on $[0,1]$, and let agent $m$ sample its keys $\bm{v}_{m}$ and $\bm{v}'_{m}$ from a gamma distribution $\Gamma(2,1)$. Let $\pi$ be some large prime number and let $a$ be a multiple of $\pi$. Then, for: \vspace{-0.2cm}
	\begin{align}
		\bm{v}_{\ell m} &= a\: e^{-\bm{v}_{\ell} \bm{v}_{m} } \mod \pi, \\
		\bm{v}'_{\ell m} &= a\: e^{-\bm{v}'_{\ell} \bm{v}'_{m} } \mod \pi,
	\end{align} \vspace{-0.2cm}
	the desired Laplacian noise can be constructed as:
	\begin{equation}\label{eq:locGHPnoise}
		\bm{g}_{\{\ell m\}k,i} = \frac{\sqrt{2}}{\sigma_g} \ln \left( \frac{\bm{v}_{\ell m}}{\bm{v}'_{\ell m}}\right).
	\end{equation}
\qed
\end{definition}

The local graph-homomorphic process proposed here is related to methods that fall under secure aggregation (see  \cite{bonawitz2016practical}). However, the main difference between our method and earlier investigations is that we devise a scheme for the more general \textit{distributed} setting, while other works focus largely on the particular case of federated learning with its specialized structure with a central processor. Furthermore, while we generate random numbers making our scheme more secure, the work \cite{bonawitz2016practical} adds pseudo-random numbers to the shared messages. Since pseudo-random numbers are generated by deterministic algorithms, it makes the noise predictable and scucebtible to attacks, contrary to random numbers. Furthermore, we quantify the privacy of our scheme as opposed to \cite{bonawitz2016practical}. In the next theorem, we show that using construction \eqref{eq:locGHPnoise} results in a differentially private algorithm.

\begin{theorem}[\textbf{Privacy of distributed learning with local graph-homomorphic perturbations}]\label{thrm:priv}
	Under the local graph-homomorphic process defined by \eqref{eq:locGHPnoise} the resulting privatized algorithm is $\epsilon(i)-$differentially private with high probability and with $\epsilon(i) $ defined in \eqref{eq:eps}. 
\end{theorem}
\begin{proof}
	See Appendix \ref{app:thrm-priv}. 
\end{proof} 
Note that since the noises cancel out locally during each iteration, the algorithm follows the same trajectory as the non-privatized algorithm. This implies that the MSD performance of the privatized distributed learning algorithm with the local graph-homomorphic perturbation \eqref{eq:locGHPnoise} is the same as the non-privatized distributed learning algorithm. In particular, the results on the convergence of the non-privatized algorithm from Theorem 9.1 in \cite{sayed2014adaptation} will continue to hold. This implies that the MSE will now be on the order of $O(\mu)\sigma_g^2$. 

	The above result highlights one difficulty with differential privacy. Note that, in principle, as the variance $\sigma_g^2$ of the added noise is increased, the level of privacy is also increased.  However, this process introduces an additional communication cost. For example, agent 1 needs to communicate to its neighbourhood the splitting into the positive agents and negative agents. It will also need to communicate with almost half the neighbourhood of its neighbours to agree on an added noise $\bm{g}_{\{1m\}k,i}$ for $k \in \mathcal{N}_1$ and $m \in \mathcal{N}_k$. Thus, in total, the communication cost for agent 1 will increase by at most $|\mathcal{N}_1| + \sum_{k=2}^6|\mathcal{N}_k|/2$. The additional communication cost is not captured by the privacy measure even though it clearly affects the level of privacy\footnote{If we were to decrease the number of times a random noise is generated by the local graph-homomorphic process and instead re-use the noise, then we would be decreasing the communication cost but increasing the chance of an attacker learning the noise and unmasking the messages.}. 
	This reduces its functionality and motivates the search for other privacy metrics, such as those based on information-theoretic measures [55]. For example, one metric could be the mutual information between the original message and the perturbed shared message \cite{Poor}. Thus, if we assume the individual messages $\{\w_{k,i}\}$ are Guassian random variables with variance $\sigma_{\w}^2$, and if we perturb them with a total Guassian noise $\bm{g}_{k,i}$ of variance $\sigma_g^2$, then the mutual information is given by:
	\begin{align}
		I(\w_{k,i}; \w_{k,i} + \bm{g}_i) &= H(\w_{k,i} + \bm{g}_{k,i}) -   H(\w_{k,i} + \bm{g}_{k,i}|\w_{k,i}) \notag \\
		&= H(\w_{k,i} + \bm{g}_{k,i}) - H(\bm{g}_{k,i}) \notag \\
		&= \frac{1}{2}\log \left( 1+  \frac{\sigma_{\w}^2 }{\sigma_g^2} \right).
	\end{align}
	Obeserve again that as we increase the noise variance $\sigma_g^2$, mutual information decreases while privacy increases.
	Mutual information again fails to capture the communication cost incurred by the process. It appears that no metric capturing the communication-privacy trade-off exists as of yet in the literature. This calls for the search of a more appropriate privacy metric for secure aggregation methods.

\section{Experimental Analysis}
We run two experiments. In the first experiment we focus on a linear regression problem with simulated data. We then study a classification problem on real data.

\subsection{Generalized Distributed Privacy Learning}
For each of consensus, CTA, and ATC diffusion, we compare four algorithms: the standard distributed algorithm, the privatized algorithm with random perturbations, the privatized algorithm with graph-homomorphic perturbations, and the privatized algorithm with local graph-homomorphic perturbations. We consider a network of 30 agents (Fig. \ref{fig:genNet}) and a regularized quadratic loss function:
\begin{equation}
	\min_{w\in \mathbb{R}^2} \frac{1}{30}\sum_{p=1}^{30} \frac{1}{100}\sum_{n=1}^{100} ( \bm{d}_p(n)-\bm{u}_{p,n}^\tran w )^2 + 0.01\Vert w\Vert^2.
\end{equation}
 We generate a random dataset $\{\bm{u}_{p,n}, \bm{d}_p(n)\}_{n=1}^{100}$ as follows: we let the two-dimensional feature vector $\bm{u}_{p,n} \sim \mathcal{N}(0;R_u)$, and add noise $\bm{o}_{p}(n) \sim \mathcal{N}(0;\sigma^2_{o,p})$ such that $\bm{d}_p(n) = \bm{u}_{p,n}^\tran w^{\star} + \bm{o}_{p}(n)$, for some generative model $w^{\star} \in \mathbb{R}^2$ and randomly set variance $R_u$ and added noise variance $\sigma_{o,p}^2$. To make the data distributions non-iid, we use different noise variances $\sigma_{o,p}^2$ across the agents. The optimal global model has a closed form solution with $\widehat{R}_u$ and $\widehat{r}_{uo}$ as defined previously:
 \begin{equation}
 	w^o = (\widehat{R}_u + 0.01I)^{-1} (\widehat{R}_u  w^{\star} + \widehat{r}_{uo}).
 \end{equation}
	
We set the step-size $\mu = 0.4$, the noise variance $\sigma_g^2 = 0.01$, and run the algorithms for 1000 iterations. We repeat the experiment 20 times and plot the MSD in the log domain of the centroid model and the average of the individual MSDs:
\begin{align}
	\text{MSD}_i &\eqdef \Vert \w_{c,i} - w^o\Vert^2, \\
	\text{MSD}_{\text{avg},i} &\eqdef \frac{1}{P}\sum_{p=1}^P \Vert \w_{p,i} - w^o\Vert^2.
\end{align}
\vspace{-0.5cm}
\begin{figure}[h!]
	\centering
	\includegraphics[width=0.3\textwidth]{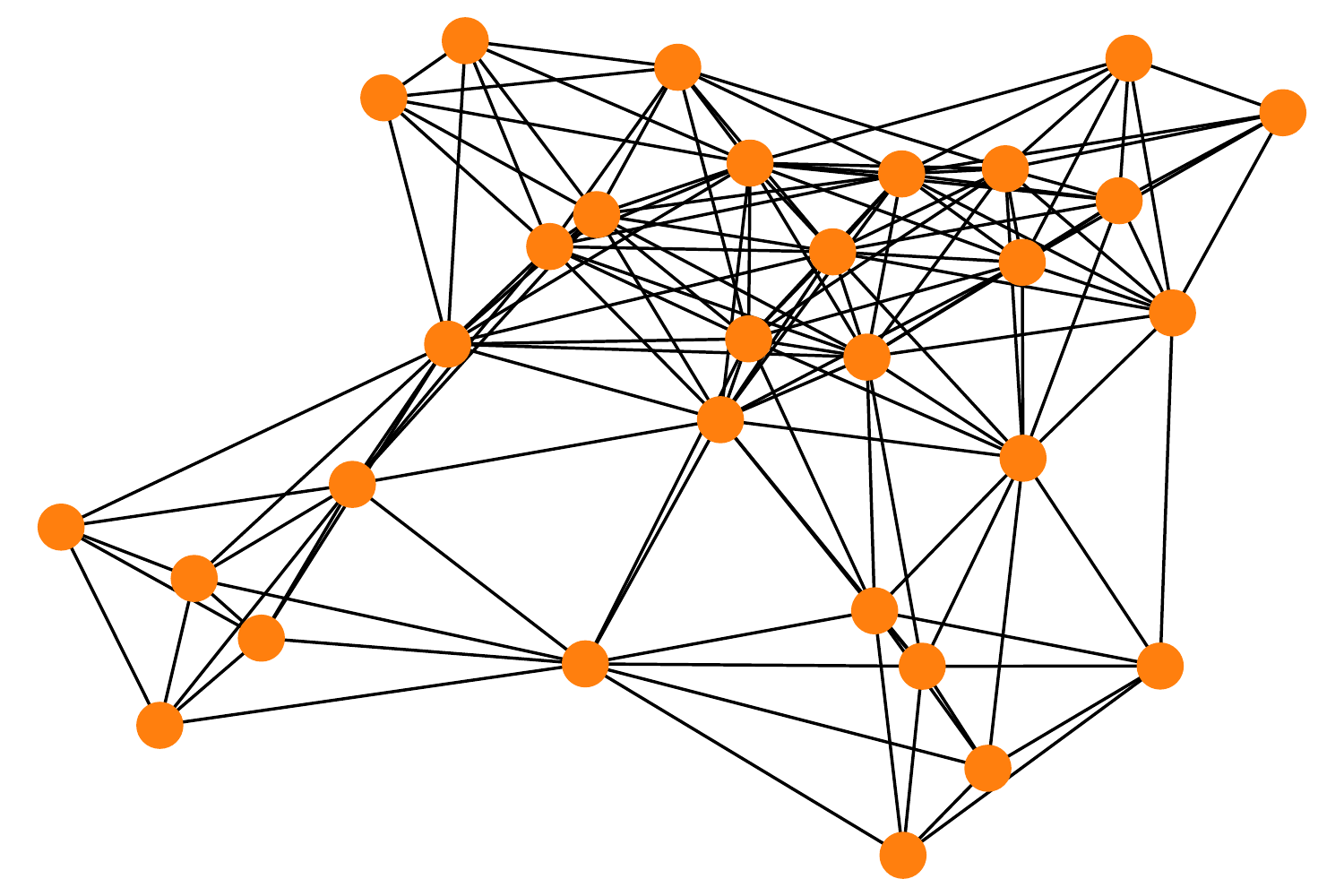}
	\caption{The generated network of agents. }\label{fig:genNet}
\end{figure}

\begin{figure*}[h!]
	\centering
	\begin{subfigure}[b]{0.32\textwidth}
		\centering
		\includegraphics[width=\textwidth]{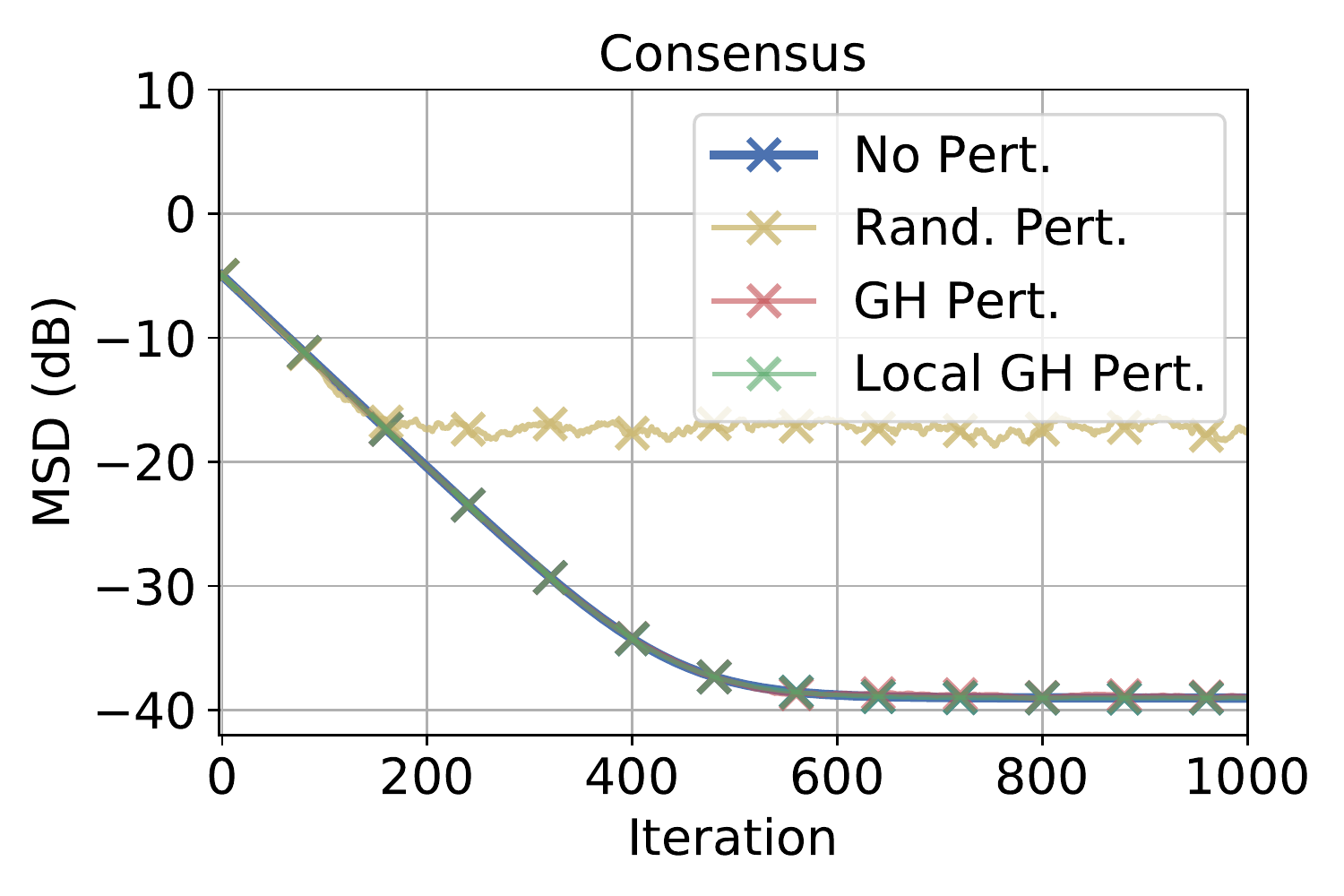}
		\caption{Consensus: centroid MSD}
	\end{subfigure}
	\begin{subfigure}[b]{0.32\textwidth}
		\centering
		\includegraphics[width=\textwidth]{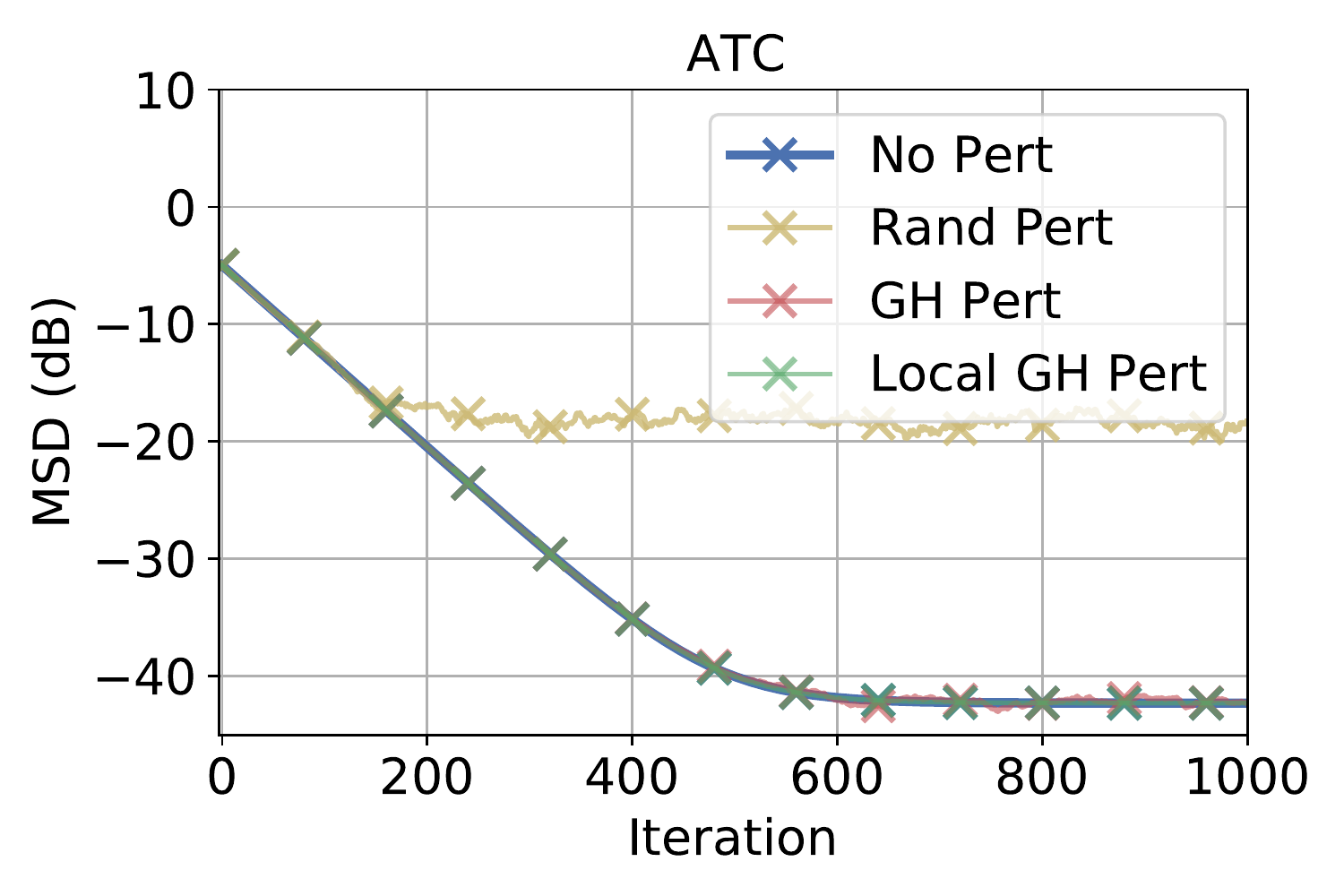}
		\caption{ATC: centroid MSD}
	\end{subfigure}
\begin{subfigure}[b]{0.32\textwidth}
	\centering
	\includegraphics[width=\textwidth]{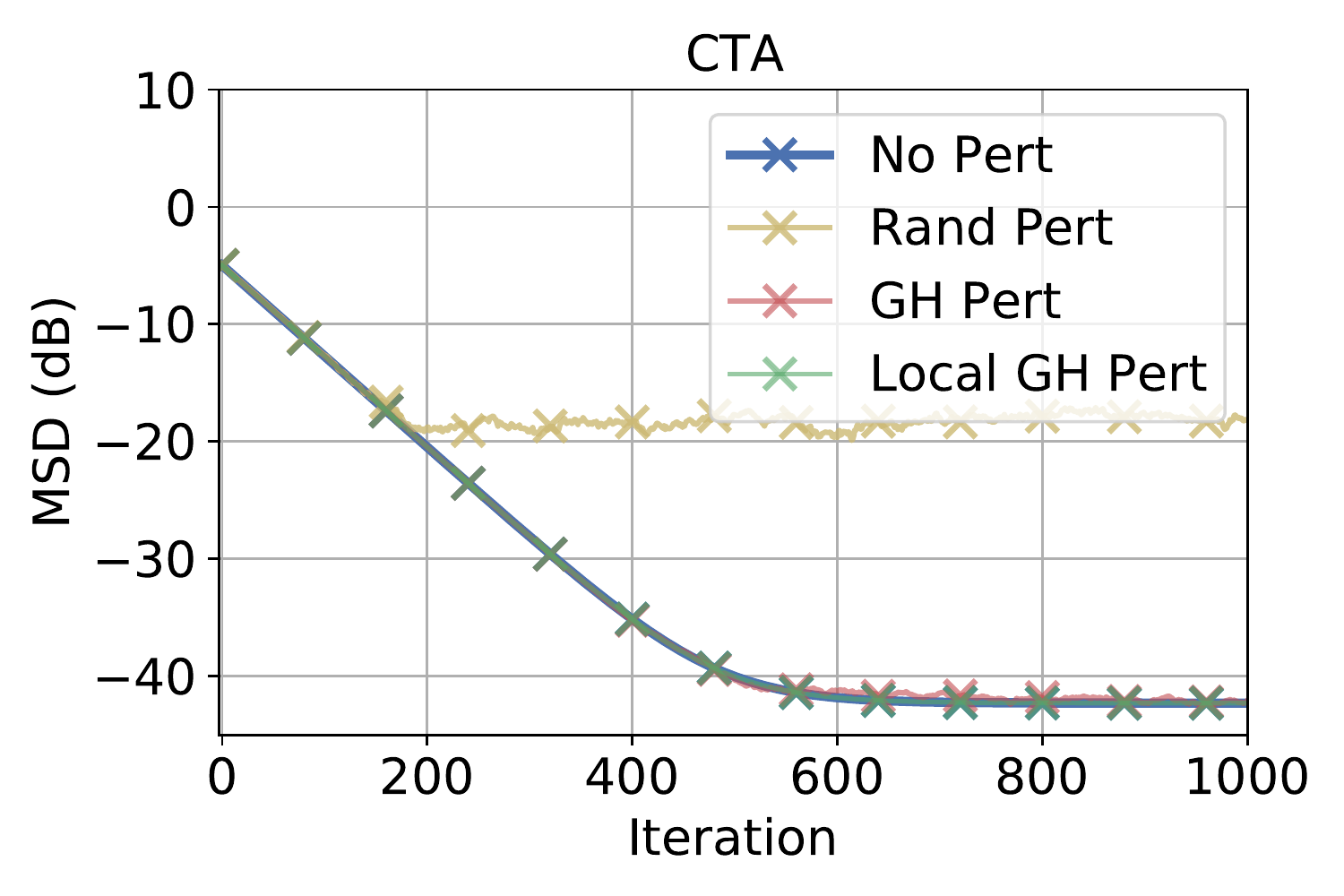}
	\caption{CTA: centroid MSD}
\end{subfigure}

	\begin{subfigure}[b]{0.32\textwidth}
		\centering
		\includegraphics[width=\textwidth]{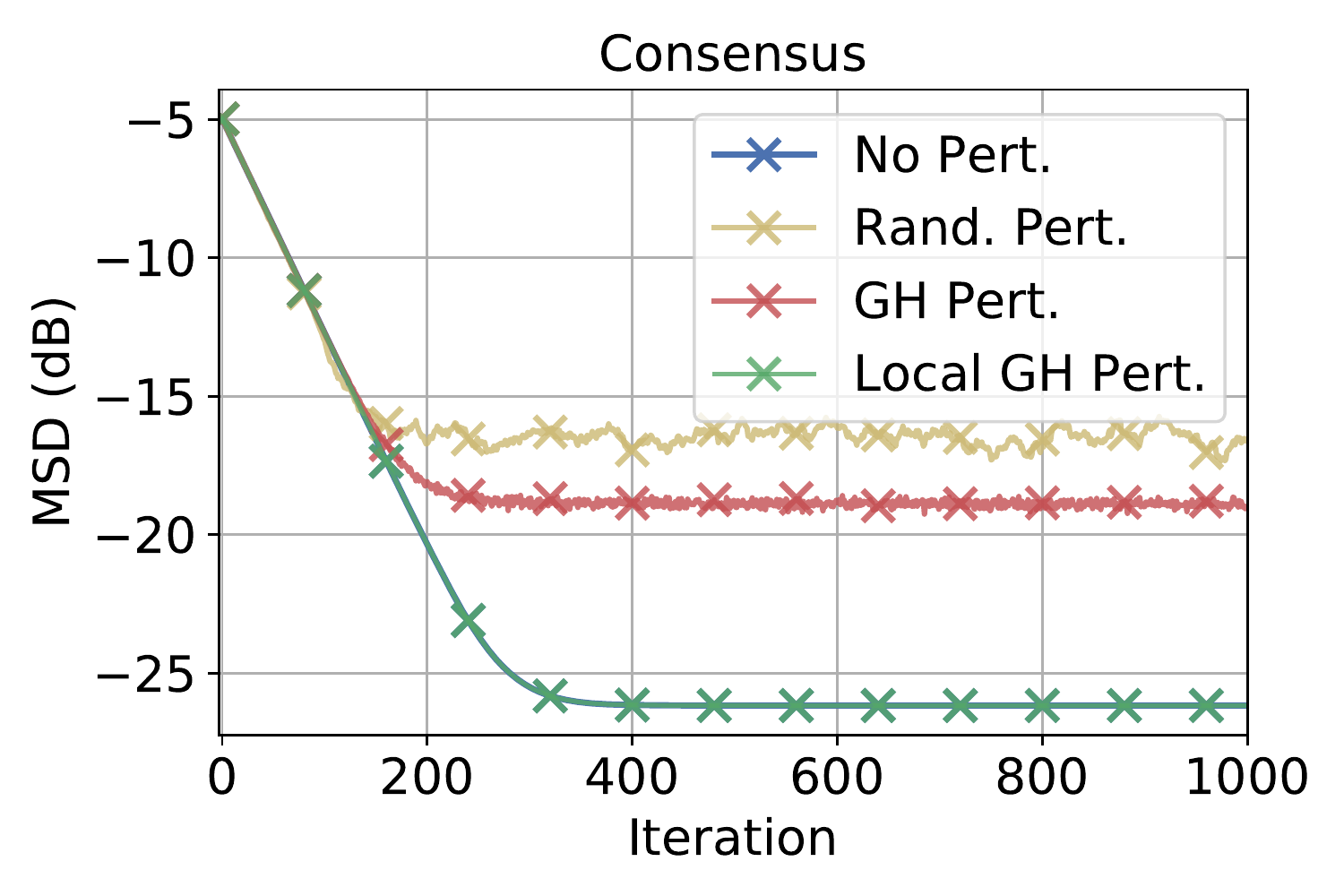}
		\caption{Consensus: avg. ind. MSD}
	\end{subfigure}	
\begin{subfigure}[b]{0.32\textwidth}
	\centering
	\includegraphics[width=\textwidth]{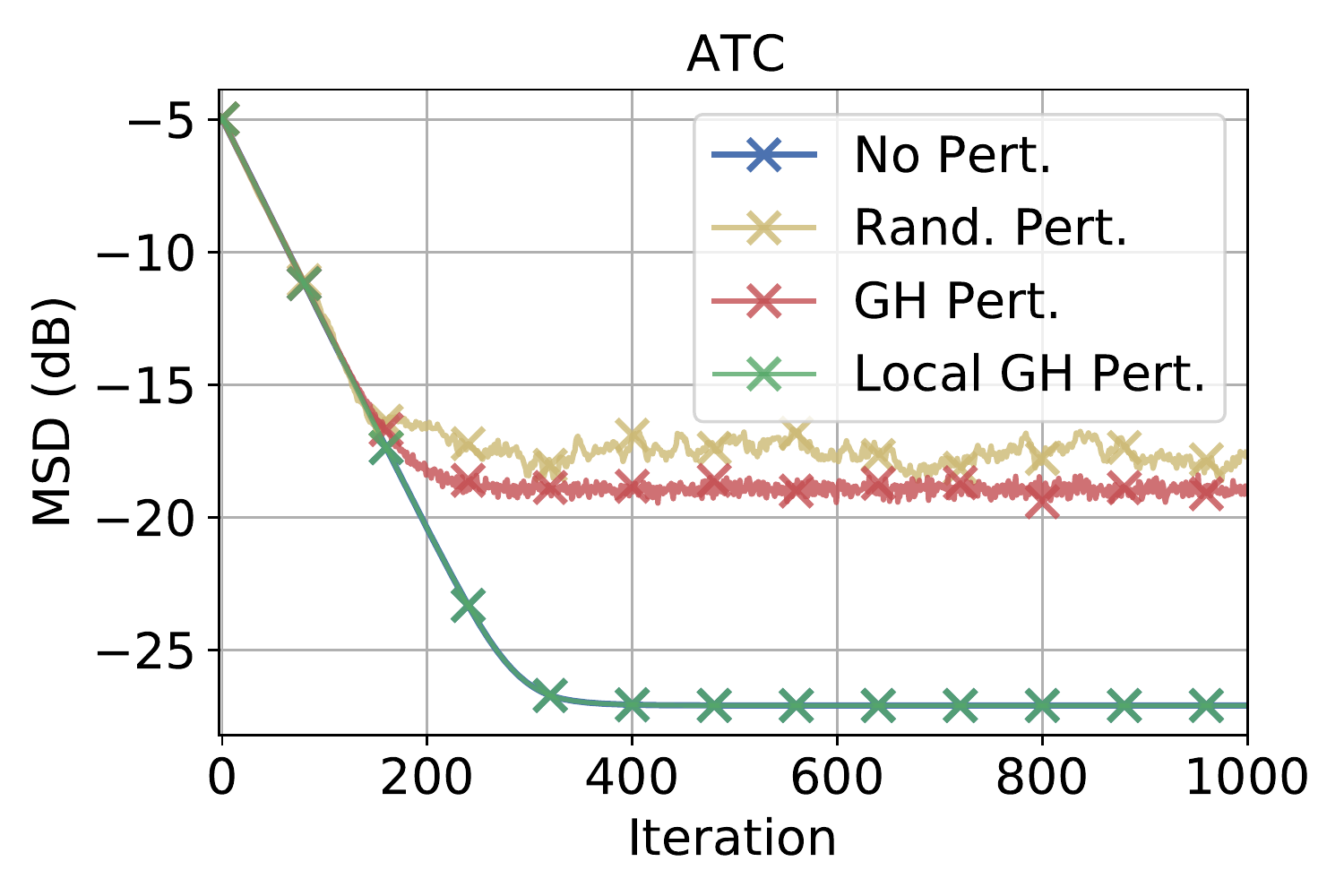}
	\caption{ATC: avg. ind. MSD}
\end{subfigure}
\begin{subfigure}[b]{0.32\textwidth}
	\centering
	\includegraphics[width=\textwidth]{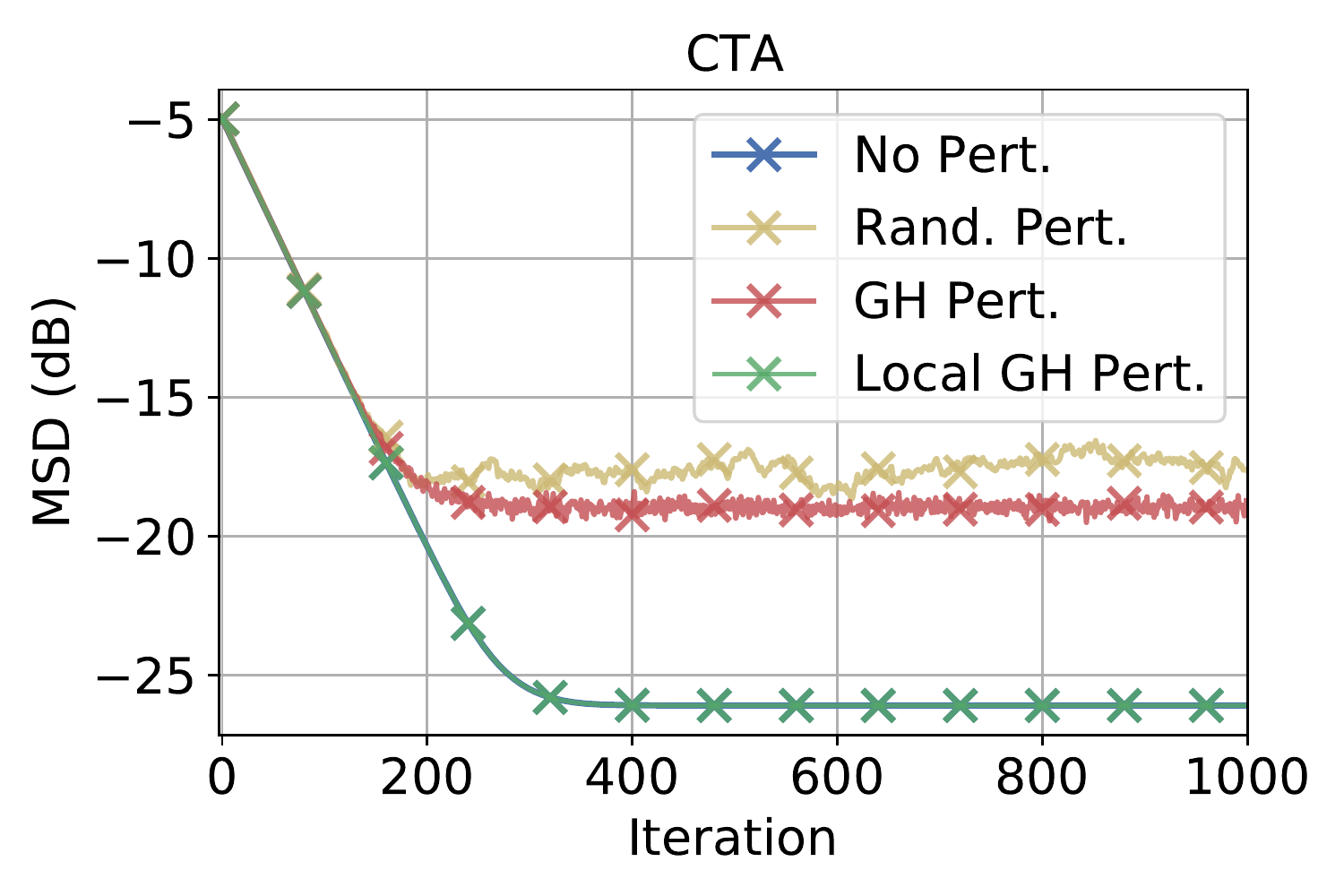}
	\caption{CTA: avg. ind. MSD}
\end{subfigure}

\caption{MSD plots for the three distributed learning algorithms.}\label{fig:MSD-dist}
\end{figure*}

As we observe in Fig. \ref{fig:MSD-dist}, graph-homomorphic perturbations do not hinder the performance of the algorithm in approximating the true model as do random perturbations. The random perturbations MSD curve (yellow) is significantly higher than the graph homomorphic perturbations MSD curve (red) which is close to the non-perturbed MSD curve (blue). If we examine the average MSD of the individual models, we observe that the decay in performance is not as much as that for random perturbations. Moreover, since local graph-homomorphic perturbations do not affect the performance of the algorithms, we observe that the MSD curve follows that of the non-privatized algorithm.



\subsection{Classification in Distributed Learning}
We next run an experiment on a classification dataset. We use the Avazu click through dataset \cite{avazu}, which contains a set of online add clicks. We distribute the data among $P=50$ agents. To get non-iid data, we add non-iid Gaussian noise to each agent's dataset. We let $\mu = 0.5$, $\rho=0.001$, and $\sigma_g^2 = 0.8$. We plot the testing error in Fig. \ref{fig:TE} of the standard algorithm, the privatized algorithm with random perturbations, the privatized algorithm with graph-homomorphic perturbations, and the privatized algorithm with local graph-homomorphic perturbations. It comes as no surprise that using graph-homomorphic perturbations does not hinder the testing error as random perturbations do, and local graph-homomorphic perturbations do not change the testing error.  
\begin{figure}[h!]
	\centering
\begin{subfigure}[b]{0.42\textwidth}
	\centering
	\includegraphics[width=\textwidth]{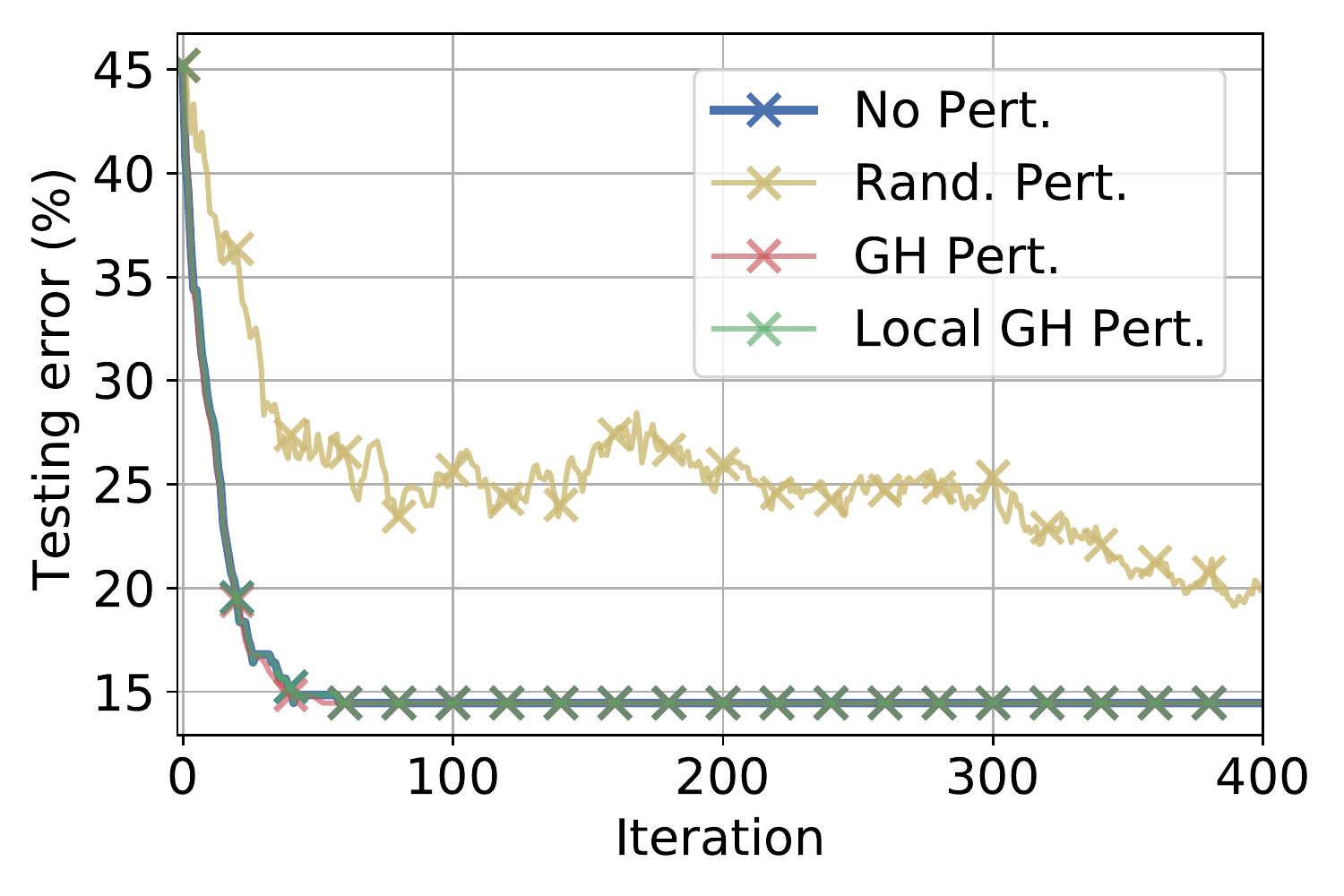}
	\caption{Centroid testing error}
\end{subfigure}
\begin{subfigure}[b]{0.42\textwidth}
	\centering
	\includegraphics[width=\textwidth]{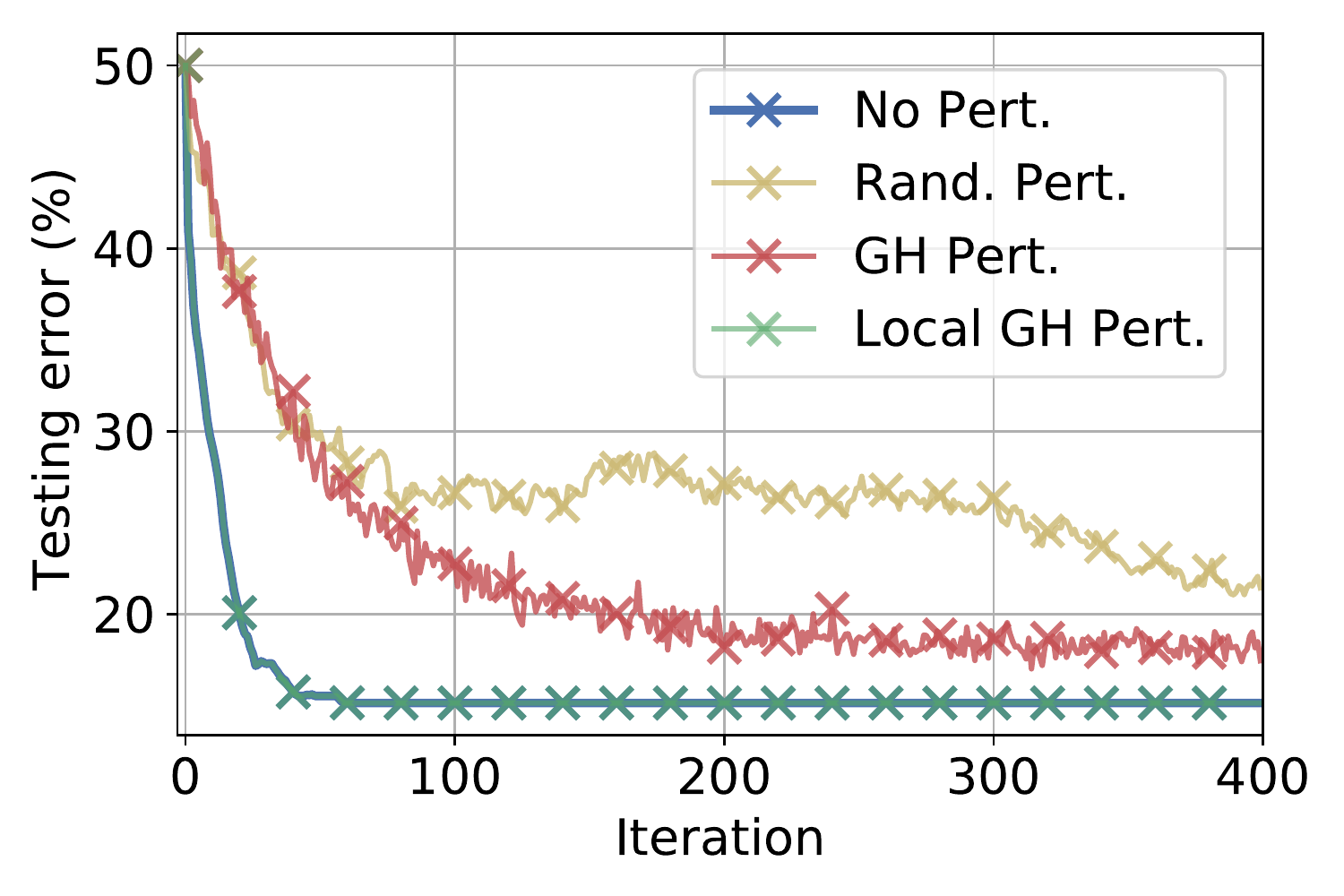}
	\caption{Average individual testing error}
\end{subfigure}
\caption{Testing error of standard ATC, privatized ATC with random perturbations, and privatized ATC with graph-homomorphic perturbations.}\label{fig:TE}
\end{figure}

\section{Conclusion}
The goal of this work has been to study the effect of privacy in distributed learning. We established the superiority of graph-homomorphic perturbations in the model performance, as opposed to random perturbations. We then designed local graph-homomorphic perturbations that ensure the added noise does not affect the model performance. Thus, the main takeaway from this work is that graph-homomorphic perturbations are better than random perturbations in distributed learning.

\appendices
\section{Sensitivity of the Distributed Algorithm}\label{app:sensCalc}
We study the sensitivity of the distributed learning algorithm \eqref{eq:preComb}--\eqref{eq:postComb}, which is defined at each time instant by the expression:
\begin{equation}
	\Delta(i) =  \Vert \ws_{i}- \ws'_{i}\Vert.
\end{equation}
This definition captures the change when the data samples of a single agent are changed. The prime symbol represents the new trajectory. We can bound the sensitivity using the triangle inequality by the individual errors and the difference in the optimal models:
\begin{equation}
	\Delta(i) \leq  \Vert \wse_{i}'\Vert + \Vert \wse_{i}\Vert + \sqrt{P}\Vert w^o - w'^o\Vert.
\end{equation}
Then, for any constants $B$ and $B'$ chosen by the desiger, we can use Markov's inequality to get the bounds:
\begin{align}
	\mathbb{P}(\Vert \wse_{i}\Vert \geq B ) &\leq \frac{\mathbb{E}\Vert \wse_i\Vert^2}{B^2}, \\
	\mathbb{P}(\Vert \wse'_{i}\Vert \geq B' ) &\leq \frac{\mathbb{E}\Vert \wse'_i\Vert^2}{B'^2}.	
\end{align}
Now we recal from Theorem \ref{thrm:MSEpriv} that:
\begin{align}
	\mathbb{E}\Vert \wse_i\Vert^2 &\leq \kappa_2^2 \mathds{1}^\tran \Gamma^i \begin{bmatrix}
			\mathbb{E}\Vert \bar{\ws}_0\Vert^2  \\ \Vert \mathbb{E}\check{\ws}_0\Vert^2 
		\end{bmatrix} + O(\mu) + O(\mu^{-1}), 
	\\
	\mathbb{E}\Vert \wse'_i\Vert^2 &\leq \kappa'^2_2 \mathds{1}^\tran \Gamma'^i \begin{bmatrix}
		\mathbb{E}\Vert \bar{\ws}'_0\Vert^2  \\ \mathbb{E}\Vert \check{\ws}'_0\Vert^2 
	\end{bmatrix} + O(\mu) + O(\mu^{-1}).
\end{align} 
It follows that the sensitivity is bounded by:
\begin{align}
	\Delta(i) \leq B + B' +\sqrt{ P} \Vert w^o - w'^o\Vert 
\end{align}
with high probability given by:
\begin{align}
		&\mathbb{P}(	\Delta(i) \leq B + B' + \sqrt{P} \Vert w^o - w'^o\Vert ) 
		\notag \\
		&\geq \left ( 1- \frac{\kappa_2^2 \mathds{1}^\tran \Gamma^i \begin{bmatrix}
			\mathbb{E}	\Vert \bar{\ws}_0\Vert^2  \\ \mathbb{E}\Vert \check{\ws}_0\Vert^2 
			\end{bmatrix} + O(\mu) + O(\mu^{-1})}{B^2} \right)
	\notag 	\\
	&\quad \times \left(1-  \frac{\kappa'^2_2 \mathds{1}^\tran \Gamma'^i \begin{bmatrix}
			\mathbb{E}	\Vert \bar{\ws}'_0\Vert^2  \\ \mathbb{E}\Vert \check{\ws}'_0\Vert^2 
		\end{bmatrix} + O(\mu) + O(\mu^{-1})}{B'^2} \right).
\end{align}

\section{Proof of Theorem \ref{thrm:MSEpriv}}\label{app:MSEpriv}
The following proof follows similar steps to those used in \cite{sayed2014adaptation} for non-private algorithms. Using the Jordan decomposition of $A_2^\tran A_0^\tran A_1^\tran$:
\begin{align}
	A_2^\tran A_0^\tran A_1^\tran &= V_{\theta} J V_{\theta}^{-1}, \\
	V_{\theta} &\eqdef \begin{bmatrix}
		q & V_R
	\end{bmatrix}, \\
	V_{\theta}^{-1}&\eqdef \begin{bmatrix}
		\mathds{1}^\tran \\ V_L^\tran
	\end{bmatrix}, \\
	J &  \eqdef  \begin{bmatrix}
		1 & 0 \\ 0 & J_{\theta}
	\end{bmatrix}, 
\end{align}
where $q$ is the Perron eigenvector of $A_2^\tran A_0^\tran A_1^\tran$ and $J_{\theta}$ contains Jordan blocks of the corresponding eigenvalues $\lambda$ of the form (example of a $3\times 3$ matrix):
\begin{equation}
	 \begin{bmatrix}
		\lambda & 0 & 0 \\
		\theta & \lambda & 0 \\
		0 & \theta & \lambda
	\end{bmatrix},
\end{equation}
with a constant $\theta$ in the subdiagonal. We first write:
\begin{align}
	\bm{\mathcal{B}}_{i-1} = (\mathcal{V}_{\theta}^{-1})^\tran (\mathcal{J} - \bm{\mathcal{D}}_{i-1}^\tran )\mathcal{V}_{\theta}^\tran,
\end{align}
where:
\begin{align}
	\mathcal{V}_{\theta}^{-1} &\eqdef  {V}_{\theta}^{-1} \otimes I_M, \\
	\mathcal{V}_{\theta} &\eqdef  {V}_{\theta} \otimes I_M, \\
	\mathcal{J} &\eqdef J \otimes I_M =  \begin{bmatrix}
		I_M & 0 \\ 0 & \mathcal{J}_{\theta}
		\end{bmatrix}, \\
	\bm{\mathcal{D}}_{i-1}^\tran &\eqdef \mu	\mathcal{V}_{\theta}^\tran \mathcal{A}_2^\tran \bm{\mathcal{H}}_{i-1} \mathcal{A}_1^\tran (\mathcal{V}_{\theta}^{-1})^\tran = \begin{bmatrix}
		\bm{D}_{11,i-1}^\tran & \bm{D}_{21,i-1}^\tran \\
		\bm{D}_{12,i-1}^\tran & \bm{D}_{22,i-1}^\tran
	\end{bmatrix}.
\end{align}
It is shown in the proof of Theorem 9.1 in \cite{sayed2014adaptation} that:
\begin{align}
	\Vert I_M - \bm{D}_{11,i-1}\Vert  &\leq 1-\sigma_{11}\mu, \\
	\Vert \bm{D}_{ij}\Vert \leq \sigma_{ij}\mu,
\end{align}
for some positive constants $\sigma_{ij}$ for $i,j=1,2$.
Multiplying both sides of the error recursion \eqref{eq:recPrivDist} from the left by $\mathcal{V}_{\theta}^\tran$:
\begin{align}
	\mathcal{V}_{\theta}^\tran \wse_i = &\: \mathcal{V}_{\theta}^\tran \bm{\mathcal{B}}_{i-1} (\mathcal{V}_{\theta}^{-1})^\tran \mathcal{V}_{\theta}^\tran \wse_{i-1} +\mu \mathcal{V}_{\theta}^\tran \mathcal{A}_2^\tran \bm{s}_i - \mu \mathcal{V}_{\theta}^\tran \mathcal{A}_2^\tran b \notag \\
	&+ \mu \mathcal{V}_{\theta}^\tran \mathcal{A}_2^\tran \bm{\mathcal{H}}_{i-1} \text{diag}( \mathcal{A}_1^\tran \bm{\mathcal{G}}_{1,i}) - \mathcal{V}_{\theta}^\tran \text{diag}(\mathcal{A}_2^\tran \bm{\mathcal{G}}_{2,i}) \notag \\
	& - \mathcal{V}_{\theta}^\tran \mathcal{A}_2^\tran \text{diag} (\mathcal{A}_0^\tran \bm{\mathcal{G}}_{0,i} )- \mathcal{V}_{\theta}^\tran \mathcal{A}_2^\tran \mathcal{A}_0^\tran \text{diag} (\mathcal{A}_1^\tran  \bm{\mathcal{G}}_{1,i}),
\end{align} 
and introducing the new notation:
\begin{align}
	\mathcal{V}_{\theta}^\tran \wse_i = \begin{bmatrix}
		(q^\tran \otimes I_M) \wse_i \\ (V_R^\tran \otimes I) \wse_i
	\end{bmatrix} &\eqdef \begin{bmatrix}
	\bar{\ws}_i \\ \check{\ws}_i 
	\end{bmatrix}, \\
	\mu\mathcal{V}_{\theta}^\tran\mathcal{A}_2^\tran \bm{s}_i = \mu\begin{bmatrix}
		(q^\tran \otimes I_M)\mathcal{A}_2^\tran \bm{s}_i \\ (V_R^\tran \otimes I) \mathcal{A}_2^\tran \bm{s}_i
	\end{bmatrix} &\eqdef \begin{bmatrix}
	\bm{\bar{s}}_i \\ \bm{\check{s}}_i
\end{bmatrix}, \\
	\mu \mathcal{V}_{\theta}^\tran\mathcal{A}_2^\tran b = \mu\begin{bmatrix}
		(q^\tran \otimes I_M)\mathcal{A}_2^\tran b\\ (V_R^\tran \otimes I) \mathcal{A}_2^\tran b
	\end{bmatrix} &\eqdef \begin{bmatrix}
		0 \\ \check{b}
	\end{bmatrix},
\end{align}
we get:
\begin{align}
	\begin{bmatrix}
		\bar{\ws}_i \\\check{\ws}_i 
	\end{bmatrix}
	= & \begin{bmatrix}
		I_M - \bm{D}_{11,i-1}^\tran & -\bm{D}_{21,i-1}^\tran \\
		- \bm{D}_{12,i-1}^\tran & \mathcal{J}_{\epsilon}
	\end{bmatrix} 
	\begin{bmatrix}
		\bar{\ws}_{i-1} \\ \check{\ws}_{i-1}
	\end{bmatrix} 
	+ \begin{bmatrix}
		\bm{\bar{s}}_{i} \\ \bm{\check{s}}_{i}
	\end{bmatrix} 
	+ \begin{bmatrix}
		0 \\ \check{b}
	\end{bmatrix} \notag \\
& + \mu \mathcal{V}_{\theta}^\tran \mathcal{A}_2^\tran \bm{\mathcal{H}}_{i-1} \text{diag} ( \mathcal{A}_1^\tran \bm{\mathcal{G}}_{1,i}) - \mathcal{V}_{\theta}^\tran  \text{diag} (\mathcal{A}_2^\tran \bm{\mathcal{G}}_{2,i}) \notag \\
& - \mathcal{V}_{\theta}^\tran \mathcal{A}_2^\tran \text{diag} (\mathcal{A}_0^\tran \bm{\mathcal{G}}_{0,i}) - \mathcal{V}_{\theta}^\tran \mathcal{A}_2^\tran \mathcal{A}_0^\tran  \text{diag} (\mathcal{A}_1^\tran  \bm{\mathcal{G}}_{1,i}).
\end{align}
Then, taking the expectation of the $\ell_2-$norm, and using Jensen's inequality, we have:
\begin{align}
	\mathbb{E}\Vert \bar{\ws}_i\Vert^2 &\leq (1-\sigma_{11} \mu) \mathbb{E}\Vert \bar{\ws}_{i-1}\Vert^2 + \frac{\sigma_{21}^2\mu}{\sigma_{11}} \mathbb{E}\Vert \check{\ws}_{i-1}\Vert^2 + \mathbb{E}\Vert \bm{\bar{s}}_i \Vert^2 
	 \notag \\
	&\quad + 2\mu^2 \mathbb{E} \Vert (q^\tran\otimes I_M) \mathcal{A}_2^\tran \bm{\mathcal{H}}_{i-1} \text{diag} ( \mathcal{A}_1^\tran \bm{\mathcal{G}}_{1,i})\Vert^2 +  
	\notag \\
	&\quad + 2\mathbb{E}\Vert  (q^\tran\otimes I_M)\mathcal{A}_2^\tran \mathcal{A}_0^\tran \text{diag}(\mathcal{A}_1^\tran \bm{\mathcal{G}}_{1,i})\Vert^2 
	\notag \\
	&\quad + \mathbb{E} \Vert (q^\tran\otimes I_M) \text{diag}( \mathcal{A}_2^\tran \bm{\mathcal{G}}_{2,i}) \Vert^2  \notag \\
	&\quad + \mathbb{E}\Vert  (q^\tran\otimes I_M) \mathcal{A}_2^\tran \text{diag} (\mathcal{A}_0^\tran \bm{\mathcal{G}}_{0,i})\Vert^2   , 
\end{align}
and:
\begin{align}
	\mathbb{E}\Vert \check{\ws}_i\Vert^2 &\leq \left( \rho(J_{\theta}) + \theta + \frac{2\sigma_{22}^2 \mu^2}{1-\rho(J_{\theta}) - \theta} \right) \mathbb{E} \Vert \check{\ws}_{i-1}\Vert^2 
	 \notag \\
	&\quad + \frac{3\sigma_{21}^2 \mu^2}{1-\rho(J_{\theta}) - \theta} \mathbb{E}\Vert \bar{\ws}_{i-1}\Vert^2  + \frac{3\Vert \check{b}\Vert^2}{1-\rho(J_{\theta}) - \theta} 
	\notag \\
	&\quad + \mathbb{E}\Vert \bm{\check{s}}_i\Vert^2 
	\notag \\
	&\quad + 2 \mu^2 \mathbb{E} \Vert (V_R^\tran \otimes I_M) \mathcal{A}_2^\tran \bm{\mathcal{H}}_{i-1} \text{diag}(\mathcal{A}_1^\tran \bm{\mathcal{G}}_{1,i})\Vert ^2 
	\notag \\
	&\quad + \mathbb{E}\Vert  J_{\epsilon}^\tran(V_R^\tran\otimes I_M) \mathcal{A}_2^\tran \mathcal{A}_0^\tran \text{diag} (\mathcal{A}_1^\tran \bm{\mathcal{G}}_{1,i})\Vert^2 
	\notag \\
	& \quad + \mathbb{E} \Vert (V_R^\tran\otimes I_M) \text{diag}(\mathcal{A}_2^\tran \bm{\mathcal{G}}_{2,i}) \Vert^2 \notag \\
	&\quad + \mathbb{E}\Vert  (V_R^\tran\otimes I_M) \mathcal{A}_2^\tran \text{diag} (\mathcal{A}_0^\tran \bm{\mathcal{G}}_{0,i})\Vert^2  ,
\end{align}
with the cross terms equal to zero due to the independence of the zero-mean random variables. Then, we bound the sum of the gradient noise:
\begin{align}
	&\mathbb{E}\Vert \bm{\bar{s}}_i\Vert^2 + \mathbb{E}\Vert \bm{\check{s}}_i \Vert^2 \notag \\
	&\leq  \Vert \mathcal{V}_{\theta}\Vert^2 \mu^2 \sum_{p=1}^P \mathbb{E}\Vert \bm{s}_{p,i}\Vert^2 \notag \\
	&\leq \kappa_1^2 \mu^2 \sum_{p=1}^P \beta_{s,p}^2 \mathbb{E}\Vert \widetilde{\bm{\phi}}_{p,i-1}\Vert^2  + \sigma_{s,p}^2 \notag \\
	&\leq \kappa_1^2\mu^2 \sum_{p=1}^P \beta_{s,p}^2  \sum_{m=1}^P \left ( \mathbb{E}\Vert \we_{m,i-1}\Vert^2 +  \mathbb{E}\Vert \bm{g}_{1,mp,i}\Vert^2 \right) + \sigma_{s,p}^2 \notag \\
	&\leq \kappa_1^2\mu^2  \beta_{s}^2\mathbb{E}\Vert \wse_{i-1}\Vert^2 + \kappa_1^2\mu^2 \sigma_s^2 + \kappa_1^2\mu^2 \sum_{p,m=1}^P\beta_{s,p}^2 \mathbb{E}\Vert \bm{g}_{1,mp,i}\Vert^2 \notag \\
	&\leq \kappa_1^2 \kappa_2^2 \mu^2 \beta_s^2 \left( \mathbb{E}\Vert \bm{\bar{s}}_i\Vert^2 + \mathbb{E}\Vert \bm{\check{s}}_i\Vert^2 \right) + \kappa_1^2 \mu^2 \sigma_s^2 
	\notag \\
	&\quad+ \kappa_1^2\mu^2 \sum_{p,m=1}^P\beta_{s,p}^2 \mathbb{E}\Vert \bm{g}_{1,mp,i}\Vert^2,
\end{align}
where we introduced the constants $\beta_s^2$ and $\sigma_s^2$, which are the sums of $\beta_{s,p}^2$ and $\sigma_{s,p}^2$, respectively.
Then, going back:
\begin{align}
	\mathbb{E}\Vert \bm{\bar{\ws}}_{i} \Vert ^2 &\leq (1-\sigma_{11}\mu + \kappa_1^2\kappa_2^2 \beta_s^2\mu^2) \mathbb{E}\Vert \bm{\bar{\ws}}_{i-1} \Vert ^2 
	\notag \\
	&\quad + \left( \frac{\sigma_{21}^2\mu}{\sigma_{11}} + \kappa_1^2\kappa_2^2 \mu^2 \right) \mathbb{E} \Vert \bm{\check{\ws}}_{i-1} \Vert^2 + \kappa_1^2\mu^2 \sigma_s^2 
	\notag \\
	&\quad +\kappa_1^2\sum_{p,m=1}^P \beta_{s,p}^2\mu^2 \mathbb{E}\Vert \bm{g}_{1,mp,i}\Vert^2\notag \\
	&\quad + 2\mu^2 \mathbb{E} \Vert (q^\tran\otimes I_M) \mathcal{A}_2^\tran \bm{\mathcal{H}}_{i-1} \text{diag}( \mathcal{A}_1^\tran \bm{\mathcal{G}}_{1,i})\Vert^2 
	 \notag \\
	&\quad + 2\mathbb{E}\Vert  (q^\tran\otimes I_M) \mathcal{A}_2^\tran \mathcal{A}_0^\tran \text{diag}(\mathcal{A}_1^\tran \bm{\mathcal{G}}_{1,i})\Vert^2 
	\notag \\
	&\quad + \mathbb{E} \Vert (q^\tran\otimes I_M) \text{diag}( \mathcal{A}_2^\tran \bm{\mathcal{G}}_{2,i})\Vert^2  
	\notag \\
	&\quad + \mathbb{E}\Vert  (q^\tran\otimes I_M) \mathcal{A}_2^\tran \text{diag} (\mathcal{A}_0^\tran \bm{\mathcal{G}}_{0,i})\Vert^2   , 
\end{align}
and:
\begin{align}
	\mathbb{E}\Vert \bm{\check{\ws}}_i\Vert^2 & \leq \left(\rho(J_{\theta}) + \theta + \frac{3\sigma_{22}^2\mu^2}{1- \rho(J_{\theta}) - \theta}  + \kappa_1^2\kappa_2^2 \beta_s^2 \mu^2\right) \mathbb{E}\Vert \bm{\check{\ws}}_{i-1}\Vert^2 \notag \\
	&\quad + \left(  \frac{3\sigma_{12}^2\mu^2}{1- \rho(J_{\theta}) - \theta} + \kappa_1^2 \kappa_2^2 \beta_s^2 \mu^2 \right)  \mathbb{E}\Vert \bm{\bar{\ws}}_{i-1} \Vert ^2 
	\notag \\
	&\quad + \frac{3\Vert \check{b}\Vert^2}{1-\rho(J_{\theta})-\theta} + \kappa_1^2\mu^2 \sigma_s^2 
	\notag \\
	&\quad 
	+ \kappa_1^2\mu^2 \sum_{p,m=1}^P\beta_{s,p}^2 \mathbb{E}\Vert \bm{g}_{1,mp,i}\Vert^2
	 \notag \\
	&\quad + 2 \mu^2 \mathbb{E} \Vert (V_R^\tran \otimes I_M) \mathcal{A}_2^\tran \bm{\mathcal{H}}_{i-1} \text{diag} (\mathcal{A}_1^\tran \bm{\mathcal{G}}_{1,i})\Vert ^2 
	\notag \\
	&\quad + \mathbb{E}\Vert  J_{\theta}^\tran(V_R^\tran\otimes I_M)  \mathcal{A}_2^\tran \mathcal{A}_0^\tran  \text{diag} (\mathcal{A}_1^\tran  \bm{\mathcal{G}}_{1,i})\Vert^2 
	\notag \\
	& \quad + \mathbb{E} \Vert (V_R^\tran\otimes I_M) \text{diag} (\mathcal{A}_2^\tran \bm{\mathcal{G}}_{2,i})\Vert^2 
	\notag \\
	&\quad + \mathbb{E}\Vert  (V_R^\tran\otimes I_M) \mathcal{A}_2^\tran \text{diag} (\mathcal{A}_0^\tran \bm{\mathcal{G}}_{0,i}) \Vert^2 .
\end{align}
Adding the two bounds:
\begin{align}
	\mathbb{E}\Vert \wse_i\Vert^2 &\leq \kappa_2^2 \bigg( \bar{\gamma} \mathbb{E}\Vert \bm{\bar{\ws}}_{i-1}\Vert^2 + \check{\gamma}\mathbb{E}\Vert \bm{\check{\ws}}_{i-1}\Vert^2 + \frac{3\Vert \check{b}\Vert^2}{1-\rho(J_{\theta}) - \theta } \notag \\
	& \quad  + 2\kappa_1^2 \mu^2 \sigma_s^2 + 2\kappa_1^2\mu^2\sum_{p,m=1}^P \beta_{s,p}^2 \mathbb{E} \Vert \bm{g}_{1,mp,i}\Vert^2 
	\notag \\
	&\quad + 2\mu^2 \mathbb{E}\Vert \mathcal{V}_{\theta}^\tran \mathcal{A}_2^\tran \bm{\mathcal{H}}_{i-1} \text{diag} (\mathcal{A}_1^\tran \bm{\mathcal{G}}_{1,i}) \Vert^2 
	\notag \\
	&\quad + \mathbb{E}\Vert \mathcal{V}_{\theta}^\tran\mathcal{A}_2^\tran \mathcal{A}_{0}^\tran \text{diag} (\mathcal{A}_1^\tran \bm{\mathcal{G}}_{1,i})\Vert^2  
	\notag \\
	& \quad + \mathbb{E}\Vert \mathcal{V}_{\theta}^\tran \text{diag }(\mathcal{A}_2^\tran \bm{\mathcal{G}}_{2,i} )\Vert^2 
	\notag \\
	&\quad
	+ \mathbb{E}\Vert \mathcal{V}_{\theta}^\tran \mathcal{A}_2^\tran \text{diag} (\mathcal{A}_0^\tran \bm{\mathcal{G}}_{0,i} ) \Vert^2 \bigg) \notag \\
	&\leq \kappa_2^2 \bigg(  \bar{\gamma} \mathbb{E}\Vert \bm{\bar{\ws}}_{i-1}\Vert^2 + \check{\gamma}\mathbb{E}\Vert \bm{\check{\ws}}_{i-1}\Vert^2 + \frac{3\Vert \check{b}\Vert^2}{1-\rho(J_{\theta}) - \theta } \notag \\
		& \quad 
	+ 2\kappa_1^2 \mu^2 \sigma_s^2 +2\kappa_1^2\mu^2 \sum_{p,m=1}^P\beta_{s,p}^2 \mathbb{E}\Vert \bm{g}_{1,mp,i}\Vert^2  
	\notag \\
	&\quad + \kappa_1^2(1+2\delta^2\mu^2)\mathbb{E}\Vert \text{diag}(\mathcal{A}_1^\tran \bm{\mathcal{G}}_{1,i}) \Vert^2 
	\notag \\
	&\quad + \kappa_1^2 \mathbb{E}\Vert \text{diag} (\mathcal{A}_2^\tran\bm{\mathcal{G}}_{2,i})\Vert^2 
	\notag \\
	&\quad + \kappa_1^2 \mathbb{E}\Vert\text{diag}(\mathcal{A}_0^\tran \bm{\mathcal{G}}_{0,i}) \Vert^2 \bigg).
\end{align}
Then, recursively bounding the MSE and taking the limit as $i$ tends to infinity, we get:
\begin{align}
	& \limsup_{i\to \infty} \mathbb{E}\Vert \wse_i\Vert^2 \notag \\
	&\leq \kappa_2^2 \mathds{1}^\tran(I  - \Gamma)^{-1}  
	\notag \\
	& \times \begin{bmatrix}
		\kappa_1^2\mu^2 \sigma_s^2 + \left(4+(2\delta^2+P\beta_s^2)\mu^2\right)\sigma_g^2 \\
		\frac{3\Vert \check{b}\Vert^2}{1-\rho(J_{\theta})-\theta} + \kappa_1^2 \mu^2 \sigma_s^2 +  \kappa_1^2 \left( 3+ (2\delta^2 + P\beta_s^2)\mu^2  \right)\sigma_g^2
	\end{bmatrix} \notag \\
&= O(\mu)\sigma_s^2 + O(\mu) + (O(\mu^{-1})+O(\mu))\sigma_{g}^2 ,
\end{align}
where:
\begin{equation}
	\Gamma \eqdef \begin{bmatrix}
		\bar{\gamma} & \frac{\sigma_{21}^2\mu}{\sigma_{11}} + \kappa_1^2\kappa_2^2 \mu^2 \\
		\frac{3\sigma_{12}^2\mu^2}{1-\rho(J_{\theta}) - \theta} + \kappa_1^2\kappa_2^2 \beta_s^2 \mu^2 & \check{\gamma} 
	\end{bmatrix}.
\end{equation}

\section{Proof of Lemma \ref{lem:netDis}}\label{app:netDis}

We define $\ws_{c,i} \eqdef (q\mathds{1}^\tran \otimes I) \ws_{i}$ and write:
\begin{align}
	\ws_i - \ws_{c,i} &= \left(I - q\mathds{1}^\tran \otimes I \right)  \ws_{i} 
	\notag \\
	&= (V_L^\tran \otimes I)(V_R \otimes I) \ws_{i}
	\notag \\
	&= (V_L^\tran\otimes I ) J_{\theta} (V_R \otimes I) \ws_{i-1} 
	\notag \\
	&\quad - \mu (V_L^\tran \otimes I)(V_R \otimes I) \col{\grad{w}J_p(\bm{\phi}_{p,i-1})}
	\notag \\
	&\quad - \mu (V_L^\tran \otimes I)(V_R \otimes I) \bm{s}_{i} 
	\notag \\
	&\quad + (V_L^\tran \otimes I)(V_R \otimes I) \big( \text{diag}(\mathcal{A}_2^\tran\bm{\mathcal{G}}_{2,i}) 
	\notag \\
	&\quad + \mathcal{A}_2^\tran \text{diag}(\mathcal{A}_0^\tran\bm{\mathcal{G}}_{0,i}) + \mathcal{A}_2^\tran \mathcal{A}_0^\tran \text{diag}(\mathcal{A}_1^\tran \bm{\mathcal{G}}_{1,i}) \big).
\end{align}
We bound $\mathbb{E}\Vert (V_R \otimes I)\ws_i\Vert^2$ by using Jensen's inequality with a constant $\rho(J_{\theta})< 1$ and define $\kappa_1^2 = \Vert \mathcal{V}_{\theta}\Vert^2$ and $\kappa_2^2 = \Vert \mathcal{V}_{\theta}^{-1}\Vert^2$:
\begin{align}\label{eq:lem1PfMainIneq}
	&\mathbb{E}\Vert (V_R \otimes I)\ws_i\Vert^2 
	\notag \\
	&\leq \rho(J_{\theta})\mathbb{E}\Vert  (V_R \otimes I)\ws_{i-1}\Vert^2 
	\notag \\
	&\quad + \frac{ \kappa_1^2 \kappa_2^2 \mu^2}{1-\rho( J_{\theta})}  \sum_{p=1}^P \mathbb{E}\Vert \bm{H}_{p,i-1}\widetilde{\bm{\phi}}_{p,i-1} + \grad{w}J_p(w^o)\Vert^2 
	\notag \\
	&\quad +\kappa_1^2 \kappa_2^2\mu^2  \sum_{p=1}^P \mathbb{E}\Vert \bm{s}_{p,i}\Vert^2 +   v_1^2 v_2^2  \big( \mathbb{E}\Vert \text{diag}(\mathcal{A}_2^\tran\bm{\mathcal{G}}_{2,i}) \Vert^2
	\notag \\
	&\quad + \mathbb{E}\Vert\mathcal{A}_2^\tran \text{diag}(\mathcal{A}_0^\tran\bm{\mathcal{G}}_{0,i}) \Vert^2 + \mathbb{E}\Vert\mathcal{A}_2^\tran \mathcal{A}_0^\tran \text{diag}(\mathcal{A}_1^\tran \bm{\mathcal{G}}_{1,i})\Vert^2 \big)
	\notag \\
	&\leq \rho(J_{\theta})\mathbb{E}\Vert  (V_R \otimes I)\ws_{i-1}\Vert^2 + \frac{2\kappa_1^2\kappa_2^2\mu^2 \Vert b\Vert^2 }{1-\rho(J_{\theta})}
	\notag \\
	&\quad + \kappa_1^2\kappa_2^2\mu^2 \sum_{p=1}^P \beta_{s,p}^2 \sum_{m\in \mathcal{N}_p}a_{1,mp}\mathbb{E}\Vert \we_{m,i-1}\Vert^2 
	\notag \\
	&\quad 	+ \kappa_1^2\kappa_2^2\mu^2 \sum_{p,m=1}^P a_{1,mp}^2 \sigma_{g}^2 + \kappa_1^2\kappa_2^2 \mu^2\sigma_s^2 
	\notag \\
	&\quad+ \kappa_1^2\kappa_2^2 \sum_{p,m=1}^P( a_{2,mp}^2 + a_{0,mp}^2  + a_{1,mp}^2 ) \sigma_{g}^2 .
\end{align}
Then, the individual errors $\mathbb{E}\Vert \we_{m,i-1}\Vert^2 $ can be bounded as shown in Theorem \ref{thrm:MSEpriv}:
\begin{align}
	\mathbb{E}\Vert \we_{m,i-1}\Vert^2 &\leq \mathbb{E}\Vert \wse_{i-1}\Vert^2 
	\notag \\
	&\leq \kappa_2^2 \mathds{1}^\tran \left( \Gamma^{i-1} \begin{bmatrix}
		\mathbb{E}\Vert \bar{\ws}_{0}\Vert^2
		\\ \mathbb{E}\Vert \check{\ws}_{0}\Vert^2
	\end{bmatrix}  \right.
	\notag \\
	&\quad \left. + (I -\Gamma)^{-1}(I-\Gamma^{i-1}) \begin{bmatrix}
		O(\mu^2) + O(1)
		\\ O(\mu^2) + O(1)
	\end{bmatrix}
	\right),
\end{align}
where the $O(\mu^2)$ and $O(1)$ terms are constants depending on the gradient noise variance, the bias term $b$, and the noise variance. Also, the matrix $\Gamma$ captures the rate of the recursion and was previously defined in Appendix \ref{app:MSEpriv}.

Then, we plug back this bound into the main inequality \eqref{eq:lem1PfMainIneq} and recursively bound over $i$. The network disagreement is then bounded as:
\begin{align}
	\frac{1}{P}\sum_{p=1}^P \mathbb{E}\Vert \w_{p,i} - \w_{c,i}\Vert^2 &\leq \frac{\kappa_2^2}{P} \mathbb{E}\Vert (V_R \otimes I)\ws_i\Vert^2,
\end{align}
and in the limit:
\begin{align}
	&\limsup_{i\to\infty} \frac{1}{P}\sum_{p=1}^P \mathbb{E}\Vert \w_{p,i} - \w_{c,i}\Vert^2  
	\notag \\
	&\leq \frac{2\kappa_1^2 \kappa_2^4 \mu^2 \Vert b\Vert^2 }{P(1-\rho(J_{\theta}))^2}  +\frac{ \kappa_1^2\kappa_2^4}{P(1-\rho(J_{\theta}))} \sigma_{g}^2  + O(\mu)
	\notag \\
	&\quad + \frac{\kappa_1^2\kappa_2^4}{P(1-\rho(J_{\theta}))} \left(O(\mu^2) \sigma_s^2  + O(\mu^2)\sigma_{g}^2  \right).
\end{align}

\section{Proof of Theorem \ref{thrm:MSE-netCent}}\label{app:thrmMSE-netCent}
Starting from \eqref{eq:netCent-err} and taking the conditional mean of the squared Euclidean norm over the past models, we can split the norm into three independent terms: the model error, the gradient noise, and the added noise. Taking again expectations and using Jensen's with $\alpha = \sqrt{1-2\nu\mu + \delta^2\mu^2}$, we have:
\begin{align}\label{eq:thrm2PfMain}
	&\mathbb{E}\Vert \we_{c,i}\Vert^2 
	\notag \\
&	\leq \alpha 
	\mathbb{E} \Vert \we_{c,i-1}  \Vert^2 + \mu^2 \mathbb{E} \left\Vert (q^\tran \otimes I)\bm{s}_i \right\Vert^2 
	\notag \\
	&\quad
	+ \frac{\mu^2}{1-\alpha} \mathbb{E}\left \Vert \sum_{p=1}^P q_p \bm{H}_{p,i-1}\sum_{m\in \mathcal{N}_p} a_{1,mp} (\w_{m,i-1} - \w_{c,i-1}) \right\Vert^2
	\notag \\ 
	&\quad
	+ \mu^2 \mathbb{E} \left \Vert \sum_{p=1}^P q_p \bm{H}_{p,i-1} \sum_{m\in \mathcal{N}_p} a_{1,mp}\bm{g}_{1,mp,i} \right\Vert^2 
	\notag \\
	&\leq \alpha 	\mathbb{E} \Vert \we_{c,i-1}  \Vert^2 + \mu^2 \mathbb{E} \left\Vert (q^\tran \otimes I)\bm{s}_i \right\Vert^2  
	\notag \\
	&\quad + \frac{\delta^2 \mu^2}{1-\alpha} \sum_{p=1}^P \mathbb{E} \Vert \w_{p,i-1} - \w_{c,i-1}\Vert^2
		\notag \\ 
	&\quad
	+ \mu^2 \mathbb{E} \left \Vert \sum_{p=1}^P q_p \bm{H}_{p,i-1} \sum_{m\in \mathcal{N}_p} a_{1,mp}\bm{g}_{1,mp,i} \right\Vert^2 .
\end{align}
We bound the gradient noise by starting from \eqref{eq:gradNoiseBD} and using Jensen's inequality to introduce $\we_{c,i-1}$:
\begin{align}
	&\mathbb{E}\left\Vert  (q^\tran \otimes I) \bm{s}_i^2 \right\Vert^2 
	\notag \\
	&=  \sum_{p=1}^P q_p^2 \mathbb{E}\Vert \bm{s}_{p,i}\Vert^2
	\notag \\
	&\leq \sum_{p=1}^P q_p^2 \beta_{s,p}^2 \mathbb{E}  \Vert \widetilde{\bm{\phi}}_{p,i-1}\Vert^2 + q_p^2 \sigma_{s,p}^2 
	\notag \\
	&\leq \sum_{p=1}^P q_p^2 \beta_{s,p}^2 \sum_{m \in \mathcal{N}_p} a_{1,mp} \mathbb{E}\Vert \we_{m,i-1}\Vert^2  
 + a_{1,mp}\mathbb{E}\Vert \bm{g}_{1,mp,i}\Vert^2  
 	\notag \\
 	&\quad + \sigma_s^2
	\notag \\
	&\leq 2 \beta_{s}^2 \mathbb{E}\Vert \we_{c,i-1}\Vert^2  + \sigma_s^2 + \beta_s^2\sigma_{g}^2
	\notag \\
	& \quad+ 2 \sum_{p=1}^P q_p^2 \beta_{s,p}^2 \sum_{m\in \mathcal{N}_p} a_{1,mp}\mathbb{E}\Vert \w_{m,i-1} - \w_{c,i-1}\Vert^2  	
	\notag \\
	&\leq 2\beta_s^2 \mathbb{E}\Vert \we_{c,i-1}\Vert^2 + \sigma_s^2 + \beta_s^2 \sigma_g^2 
	\notag \\
	&\quad + 2\beta_s^2 \sum_{p=1}^P \mathbb{E}\Vert \w_{p,i-1} - \w_{c,i-1}\Vert^2.
\end{align}
The noise term can be bounded as follows by using twice Jensen's inequality:
\begin{align}
	\mathbb{E}\left\Vert \sum_{p=1}^P q_p \bm{H}_{p,i-1} \sum_{m\in \mathcal{N}_p} a_{1,mp}\bm{g}_{1,mp,i} \right\Vert^2 &\leq \delta^2 \sigma_g^2.
\end{align} 
We plug the bounds on the gradient noise and the added privacy noise in \eqref{eq:thrm2PfMain}:
\begin{align}
	\mathbb{E}\Vert \we_{c,i}\Vert^2 &\leq 
	 \left(\alpha + 2\beta_s^2\mu^2 \right)\mathbb{E}\Vert \we_{c,i-1}\Vert^2 + \mu^2\sigma_s^2 + (\beta_s^2+\delta^2) \mu^2\sigma_{g}^2 
	\notag \\
	&\quad + \left(2\beta_s^2+ \frac{\delta^2}{1-\alpha} \right)\mu^2\sum_{p=1}^P \mathbb{E}\Vert \w_{p,i-1}-\w_{c,i-1}\Vert^2. 
\end{align}
We use the bound from Lemma \ref{lem:netDis}. Recursively bounding the second-order moment of the error and taking the limit:
\begin{align}
	\limsup_{i \to \infty} \mathbb{E}\Vert \we_{c,i}\Vert^2 \leq& \frac{\mu^2\left(\sigma_s^2 + (\beta_s^2 + \delta^2)\sigma_g^2\right)}{1-\gamma_c} + O(\mu^2) 
	\notag \\
	& + \frac{\mu}{1-\gamma_c}\left( 2\beta_s^2 + \frac{\delta^2}{1-\alpha} \right) \frac{\kappa_1^2\kappa_2^4}{1-\rho(J_{\theta})} \sigma_g^2 
	\notag \\
	&= O(\mu)\sigma_s^2 + O(1)\sigma_g^2 + O(\mu^2).
\end{align}

\section{Proof of Theorem \ref{thrm:priv}}\label{app:thrm-priv}
It suffices to show the noise generated from the local graph-homomorphic process is Laplacian since we already know that adding Laplacian noise makes the algorithm differentially private (see \cite{dwork2014algorithmic,vlaski2020graphhomomorphic}) with high probability. Thus, it is well known that the product of a uniform random variable $U(0,1)$ with a gamma random variable $\Gamma(2,1)$ results in an exponential random variable Exp$(1)$ \cite{johnson1995continuous}. Then $e^{-\bm{v}_{\ell}\bm{v}_{m}}$ is uniformly distributed on $[0,1]$:
\begin{align}
	\mathbb{P}(e^{-\bm{v}_{\ell}\bm{v}_{m}} \leq c) &= \mathbb{P}(\bm{v}_{\ell}\bm{v}_{m} \geq -\ln c) = e^{\ln c} = c.
\end{align}
But multiplying it by $a$ makes the resulting variable uniformly distributed on $[0,a]$. The modulo $p$ of a uniform random variable is uniform on $[0,\pi]$ so long as $a$ is a multiple of $\pi$. Let $a = t\pi$ for some integer $t$ and $\bm{x} \sim \text{U}(0,a)$. We divide the intervel $[0,a]$ into $t$ disjoint sub-intervals of length $\pi$, $[0,a] = [0,1) \cup [1,2) \cdots \cup [(t-1)\pi, a]$. On each of these sub-intervals $[i\pi, (i+1)\pi)$, $\bm{x}$ is uniformly distributed $\mathbb{P} \big(\bm{x} \leq x | \bm{x} \in [i\pi, (i+1)\pi) \big) = x$, and so will $\bm{x} \mod \pi = \bm{x} - \lfloor \bm{x}/\pi \rfloor = \bm{x} - i\pi$ on $[0,\pi]$. Thus since $a = t\pi$ we get:
\begin{align}
	\mathbb{P}(\bm{x} \leq x) &= \sum_{i=0}^{t-1}\mathbb{P} \big(\bm{x} \leq x | \bm{x} \in [i\pi, (i+1)\pi) \big) \notag \\
	&\qquad \times  \mathbb{P} \big(\bm{x} \in [i\pi, (i+1)\pi) \big) \notag \\
	&= \sum_{i=0}^{t-1} x \frac{\pi}{a} = x.
\end{align}
This now means that $\bm{v}_{\ell m} \sim \text{U}(0,\pi)$. Then, taking the difference of two exponential random variables results in a Laplacian. Thus, we require to transform two uniform random variables to two exponential random variables with parameter $\frac{\sigma_g}{\sqrt{2}}$. Taking $-\frac{\sqrt{2}}{\sigma_g} \ln \bm{v}_{\ell m}$ results in an exponential random variable: 
\begin{align}
	\mathbb{P} \left(-\frac{\sqrt{2}}{\sigma_g}\ln(\bm{v}_{\ell m}) \leq c \right) &= \mathbb{P}\left(\bm{v}_{\ell m} \geq e^{-\frac{c \sigma_g} { \sqrt{2}}} \right ) = 1 - e^{-\frac{ \sigma_g c} {\sqrt{2}}}.
\end{align}  

%
%



\end{document}